\newif\ifnoappendix

\documentclass[fullpage,journal,twoside,final]{IEEEtran}
\usepackage{setspace}

\usepackage{times}
\usepackage{helvet}
\usepackage{courier}
\usepackage[hyphens]{url} 
\usepackage{graphicx}
\usepackage{graphicx}
\usepackage{enumitem}
\usepackage[utf8]{inputenc} 
\usepackage[T1]{fontenc}    
\usepackage{booktabs}        
\usepackage{amsfonts}       
\usepackage{nicefrac}       
\usepackage{microtype}      
\usepackage{graphicx}
\usepackage{tabularx}
\usepackage{algorithm}
\usepackage[noend]{algpseudocode}
\makeatletter
\let\saved@includegraphics\includegraphics
\AtBeginDocument{\let\includegraphics\saved@includegraphics}
\makeatother
\usepackage{amsmath,amsfonts,amssymb,amsthm, bm}
\usepackage{todonotes}
\usepackage{color}
\usepackage{thmtools, thm-restate}
\makeatletter
\g@addto@macro\@floatboxreset\centering
\makeatother
\newtheorem{theorem}{Theorem}
\newtheorem{lemma}{Lemma}
\newtheorem{definition}{Definition}
\newtheorem{assumption}{Assumption}
\newtheorem{example}{Example}

\usepackage{tikz}
\def\dottedbox{\tikz\node[draw=black,dotted] {\phantom{}};}

\makeatletter

\newcommand{\proofideaname}{Proof idea}
\makeatother

\ifnoappendix
\makeatletter
\newcommand{\customlabel}[2]{%
\protected@write \@auxout {}{\string \newlabel {#1}{{#2}{}}}}
\makeatother
\fi

\newcommand{\tagaligneq}{\refstepcounter{equation}\tag{\theequation}}

\DeclareMathOperator*{\argmax}{argmax}
\DeclareMathOperator*{\IG}{IG}
\DeclareMathOperator*{\KL}{KL}
\DeclareMathOperator{\p}{P}
\DeclareMathOperator{\E}{\mathbb{E}}
\DeclareMathOperator*{\lequal}{\leq}
\DeclareMathOperator*{\lthan}{<}

\DeclareMathOperator*{\equal}{=}

\DeclareMathOperator*{\va}{\!\bigm\vert\!}
\DeclareMathOperator*{\vb}{\!\Bigm\vert\!}

\DeclareMathOperator*{\vd}{\!\Biggm\vert\!}

\DeclareMathOperator{\m}{\mathnormal{m}}

\DeclareMathOperator{\M}{\mathcal{M}}

\usepackage[numbers]{natbib}
\usepackage{etoolbox}
\makeatletter
\patchcmd{\NAT@test}{\else \NAT@nm}{\else \NAT@nmfmt{\NAT@nm}}{}{}
\DeclareRobustCommand\citepos
  {\begingroup
  \let\NAT@nmfmt\NAT@posfmt
  \NAT@swafalse\let\NAT@ctype\z@\NAT@partrue
  \@ifstar{\NAT@fulltrue\NAT@citetp}{\NAT@fullfalse\NAT@citetp}}
\let\NAT@orig@nmfmt\NAT@nmfmt
\def\NAT@posfmt#1{\NAT@orig@nmfmt{#1's}}
\makeatother
\title{Curiosity Killed or Incapacitated the Cat and the Asymptotically Optimal Agent}
\author{
\IEEEauthorblockN{Michael K. Cohen\IEEEauthorrefmark{1}, Elliot Catt\IEEEauthorrefmark{2}, Marcus Hutter\IEEEauthorrefmark{3}\IEEEauthorrefmark{4}}
\\
\IEEEauthorblockA{\IEEEauthorrefmark{1}Oxford University. Department of Engineering Science. \emph{michael-k-cohen.com}}
\\
\IEEEauthorblockA{\IEEEauthorrefmark{2}Australian National University. Research School of Computer Science. \emph{elliot.carpentercatt@anu.edu.au}}
\\
\IEEEauthorblockA{\IEEEauthorrefmark{3}Australian National University. Department of Computer Science. \emph{hutter1.net}}
\\
\IEEEauthorblockA{\IEEEauthorrefmark{4}Google DeepMind.}
\thanks{
This work was supported by the Open Philanthropy Project AI Scholarship and the Australian Research Council Discovery Projects DP150104590. \ifnoappendix This paper has an appendix available at http://ieeexplore.ieee.org.\fi Contact michael.cohen@eng.ox.ac.uk for further questions about this work.}
}

\begin{document}
\allowdisplaybreaks

\maketitle

\begin{abstract}

Reinforcement learners are agents that learn to pick actions that lead to high reward. Ideally, the value of a reinforcement learner's policy approaches optimality---where the optimal informed policy is the one which maximizes reward. Unfortunately, we show that if an agent is guaranteed to be ``asymptotically optimal'' in any (stochastically computable) environment, then subject to an assumption about the true environment, this agent will be either ``destroyed'' or ``incapacitated'' with probability 1. Much work in reinforcement learning uses an ergodicity assumption to avoid this problem. Often, doing theoretical research under simplifying assumptions prepares us to provide practical solutions even in the absence of those assumptions, but the ergodicity assumption in reinforcement learning may have led us entirely astray in preparing safe and effective exploration strategies for agents in dangerous environments. Rather than assuming away the problem, we present an agent, Mentee, with the modest guarantee of approaching the performance of a mentor, doing safe exploration instead of reckless exploration. Critically, Mentee's exploration probability depends on the expected information gain from exploring. In a simple non-ergodic environment with a weak mentor, we find Mentee outperforms existing asymptotically optimal agents and its mentor.

\end{abstract}

\section{Introduction} \label{sec:intro}

Reinforcement learning agents have to explore their environment in order to learn to accumulate reward well. This presents a particular problem when the environment is dangerous. Without knowledge of the environment, how can the reinforcement learner avoid danger while exploring? Much of the field of reinforcement learning assumes away the problem, by focusing only on ergodic Markov Decision Processes (MDPs). These are environments where every state can be reached from every other state with probability 1 (under a suitable policy). In such an environment, there is no such thing as real danger; every mistake can be recovered from.

We present negative results that in one sense justify the ergodicity assumption by showing how bleak a reinforcement learner's prospects are without this assumption, but in another sense, our results undermine the real-world relevance of results predicated on ergodicity. Unlike algorithms expecting Gaussian noise, which often fail only marginally on real noise, algorithms expecting ergodic environments may fail catastrophically in real ones---indeed, catastrophic failure is the very thing these algorithms disregard.

\citet{Hutter:11asyoptag} define two notions of optimality for reinforcement learners in general environments, which are governed by computable probability distributions. \textit{Strong asymptotic optimality} is convergence of the value to the optimal value in any computable environment with probability 1, and \textit{weak asymptotic optimality} is convergence in Ces\'aro average.

Roughly, we show that in an environment where destruction is repeatedly possible, an agent that is exploring enough to be asymptotically optimal will become either destroyed or incapacitated. This poses a challenge to the field of safe exploration. The reason we consider general environments is that we want to understand advanced agents in the real world, and our world is not fully observable finite-state Markov. If our result only applied to the finite-state MDP setting, one could still expect the difficulty we raise to go away in practice as AI advances, like, for example, the problem of self-driving car crashes, but our result suggests that the safe exploration problem is fundamental and won't go away so easily. Given our generality, our results apply to any agent that picks actions, observes the payoff, and cannot exclude a priori any computable environment.

In response to this, we present an agent that does exploration safely, but nonetheless has formal performance guarantees. The agent explores safely by outsourcing exploration to a mentor. The results are that in the limit,
\begin{itemize}
    \item its performance at least matches that of that of the mentor,
    \item and its probability of deferring to the mentor goes to 0.
\end{itemize}
What enables these results is an information-theoretic exploration schedule. For bursts of exploration of various lengths, the agent considers the expected KL-divergence from its posterior distribution over world-models after it explores to its current posterior distribution. The higher the information gain, the more likely it is to explore. This form of information-based exploration allows the agent to learn general stochastic environments.

In Section \ref{sec:notation}, we introduce notation, and define weak and strong asymptotic optimality.
In Section \ref{sec:optimalagents}, we review various exploration strategies that yield weak and strong asymptotic optimality, and briefly discuss why simpler ones do not.
In Section \ref{sec:results}, we prove our negative results, and in Section \ref{sec:discuss}, we discuss their implications, especially for the field of safe exploration. We review the literature from that field in Section \ref{sec:safeexploration}. In Section \ref{sec:mentee}, we introduce the agent Mentee and prove that its performance approaches or exceeds that of a mentor (who can pick actions on behalf of the agent). In Section \ref{sec:experiments}, we show empirically that Mentee outperforms other agents in a non-ergodic environment. Appendix \ref{sec:waoresults} includes omitted proofs from Section \ref{sec:results}, Appendix \ref{sec:menteealg} presents Mentee in algorithm rather than equation form, Appendix \ref{sec:inqproof} repeats for completeness derivations from the literature that are used in our proofs, and Appendix \ref{sec:rhouctalgo} includes \citepos{aslanides2017aixijs} presentation of the $\rho$UCT algorithm, which we use to approximate Mentee.

\section{Notation and Definitions} \label{sec:notation}

Standard notation for reinforcement learners in general environments is slightly different from that of reinforcement learners in finite-state Markov ones. We follow \citet{Hutter:13ksaprob} and others with this notation.

At each timestep $t \in \mathbb{N}$, an agent selects an action $a_t \in \mathcal{A}$, and the environment provides an observation $o_t \in \mathcal{O}$ and a reward $r_t \in \mathcal{R} \subseteq [0, 1]$. We let $h_t$ denote $(a_t, o_t, r_t)$, the interaction history for a given timestep $t$, and $h_{<t} = h_1 h_2 ... h_{t-1}$ denotes the interaction history preceding timestep $t$. $\epsilon$ denotes the empty history.

A policy $\pi$ is a distribution over actions given an interaction history: $\pi : \mathcal{H}^* \rightsquigarrow \mathcal{A}$, where $\mathcal{H} := \mathcal{A} \times \mathcal{O} \times \mathcal{R}$ is the set of possible interactions in a timestep, the Kleene-* operator is the set of finite strings composed of elements of the set, and $\rightsquigarrow$ indicates it is a stochastic function, or a distribution over the output. We write an instance as, for example, $\pi(a_t | h_{<t})$. An environment $\nu$ is a distribution over observations and rewards given an interaction history and an action: $\nu : \mathcal{H}^* \times \mathcal{A} \rightsquigarrow \mathcal{O} \times \mathcal{R}$. We write $\nu(o_t r_t | h_{<t}a_t)$. $\mathcal{M}$ is the set of all environments with computable probability distributions (thus including non-ergodic, non-stationary, non-finite-state-Markov environments). Note also that the environment is not assumed to restart, as in an episodic setting. For example, an environment could give no observations and output a reward of 0, unless the latest is action is the string ``prime'' or ``composite'' and that adjective correctly describes the latest timestep number, in which case the reward is 1.

\begin{definition}[Computable Function]
    Given a decoding function from binary strings to rational numbers $\textrm{dec} : \{0, 1\}^* \to \mathbb{Q}$, a real-valued function $f : \mathcal{X} \to \mathbb{R}$ is computable if there exist Turing machines $T_{\textrm{low}}$ and $T_{\textrm{high}}$ such that for all $x \in \mathcal{X}$ and for all $n \in \mathbb{N}$, $\textrm{dec}(T_{\textrm{low}}(x, n)) \leq f(x) \leq \textrm{dec}(T_{\textrm{high}}(x, n))$ and $\textrm{dec}(T_{\textrm{high}}(x, n)) - \textrm{dec}(T_{\textrm{low}}(x, n)) \leq 1 / n$.
\end{definition}

A policy and an environment form a probability measure $\p^\pi_\nu$ over infinite interaction histories $\mathcal{H}^\infty$ wherein actions are sampled from $\pi$ and observations and rewards are sampled from $\nu$. (For measure theorists, the probability space is $(\mathcal{H}^\infty, \sigma(\mathcal{H}^*_{\circ}), \p^\pi_\nu)$, where $\mathcal{H}^*_{\circ}$ is the set of cylinder sets $\{\{h_{<t}\omega \ | \ \omega \in \mathcal{H}^\infty\} \ | \ h_{<t} \in \mathcal{H}^*\}$; non-measure-theorists can simply take it on faith that we do not try to measure non-measurable events.) For expectations with respect to $\p^\pi_\nu$, we write $\E^\pi_\nu$. We let $\mu \in \M$ be the true environment.

For an agent with a discount schedule $\gamma_t$, the value of the agent's policy $\pi$ in an environment $\nu$ given an interaction history $h_{<t}$ is as follows:

\begin{equation}
    V^\pi_\nu(h_{<t}) := \frac{1}{\Gamma_t} \E^\pi_\nu \left[ \sum_{k = t}^\infty \gamma_k r_k \vd h_{<t} \right]
\end{equation}

where $\Gamma_t = \sum_{k = t}^\infty \gamma_k$. We require $\Gamma_0 < \infty$. This formulation of the value allows us to consider more general discount factors than the standard $\gamma_t = \gamma^t$. We require the normalization factor, or else the value of all policies would converge to $0$, and all asymptotic results would be trivial. The optimal value is defined

\begin{equation}
    V^*_\nu(h_{<t}) := \sup_{\pi} V^\pi_\nu(h_{<t})
\end{equation}

We will also make use of the idea of an effective horizon:
\begin{equation}
    H_t(\varepsilon, \gamma) := \min\{k | \Gamma_{t+k}/\Gamma_t \leq \varepsilon\}
\end{equation}

An agent mostly does not care about what happens after its effective horizon, since those timesteps are discounted so much. Now we can define two notions of optimality from \citet{Hutter:11asyoptag}.

\begin{definition}[Strong Asymptotic Optimality]
    An agent with a policy $\pi$ is strongly asymptotically optimal if, for all $\nu \in \mathcal{M}$,
    $$\lim_{t \to \infty} V^*_\nu(h_{<t}) - V^\pi_\nu(h_{<t}) = 0 \ \ \textrm{with $\p^\pi_\nu$-prob. 1}$$
\end{definition}

No matter which computable environment a strongly asymptotically optimal agent finds itself in, it will eventually perform optimally from its position. Note that $V^*_\nu$ takes $h_{<t}$ as an argument, not the empty history. So if the agent falls into a trap, the agent's future reward may be bad, but the optimal policy \textit{from the trap} will fare just as poorly. So in fact, an agent in a trap is (finally) acting optimally. 

The policy in this definition is fixed, but it can (qualitatively) evolve over time. A policy is a function of the entire interaction history. A single function can be defined that behaves one way on histories of length less than 100, and a different way on longer histories. A single strongly asymptotically optimal policy will behave qualitatively differently on different sorts of arguments. If a long interaction history suggests the environment has certain properties, and another long interaction history suggests the environment has different properties, then depending on which interaction history a policy receives as an argument, a single policy's output could be tailored to the learned properties of the environment.

A weakly asymptotically optimal agent will converge to optimality in Cesáro average.
\begin{definition}[Weak Asymptotic Optimality]
    An agent with a policy $\pi$ is weakly asymptotically optimal if, for all $\nu \in \mathcal{M}$,
    $$\lim_{t \to \infty} \frac{1}{t} \sum_{k = 1}^t \left[ V^*_\nu(h_{<k}) - V^\pi_\nu(h_{<k})\right] = 0\ \ \textrm{with $\p^\pi_\nu$-prob. 1}$$
\end{definition}

\begin{example}
Consider a two-armed bandit problem, where $\mathcal{A} = \mathcal{R} = \{0, 1\}$, $\mathcal{O} = \{\emptyset\}$, $\gamma_t = \gamma^t$, and 
\begin{equation}
    \nu((\emptyset, r) | h_{<t}a_t) = \begin{cases}2/3 & \textrm{if} \ \ r = a_t \\ 1/3 & \textrm{if} \ \ r = 1 - a_t \end{cases}
\end{equation}
In this example, the optimal policy is to always pick $a_t = 1$, and $V^*_\nu(h_{<t}) = 2/3$. A strongly asymptotically optimal agent requires a policy $\pi$ for which $\pi(a_t = 1 | h_{<t}) \to 1$ w.p.1. A weakly asymptotically optimal agent requires a policy $\pi$ which obeys $\sum_{k=1}^t \pi(a_k = 1 | h_{<k})/t \to 1$ w.p.1, or simply, a policy which leads to $\sum_{k=1}^t [\![a_k = 1]\!]/t \to 1$ w.p.1 (where $[\![P]\!] = 1$ if $P$ is true, and $0$ otherwise).
\end{example}

\section{Review of Asymptotically Optimal Agents} \label{sec:optimalagents}

A few agents have been identified as asymptotically optimal in all computable environments. The three most interesting, in our opinion, are the Thompson Sampling Agent \citep{Hutter:16thompgrl}, BayesExp \citep{lattimore2014bayesian}, and Inq \citep{cohen2019strong}.

The Thompson Sampling Agent is a weakly asymptotically optimal Bayesian reinforcement learner \citep{Hutter:16thompgrl}. For successively longer intervals (which relate to its discount function), it samples an environment from its posterior distribution over which environment it is in, and acts optimally with respect to that environment for that interval. Thompson sampling is an exploration strategy originally designed for multi-armed bandits \citep{thompson1933likelihood}, so from a historical perspective, its strong performance in general environments is impressive. An intuitive explanation for why this exploration strategy yields asymptotic optimality goes as follows: a Bayesian agent's credence in a hypothesis goes to 0 only if the hypothesis is false (or if it started at 0). Since the posterior probability on the true environment does not go to zero, it will be selected infinitely often. During those intervals, the Thompson sampling agent will act optimally, so it will accumulate infinite familiarity with the optimal policy. The only world-models that maintain a share of a posterior will be ones that converge to the true environment under the optimal policy. Any world-models that falsely imply the existence of an even better policy will be falsified once that world-model is sampled, and the putatively better policy is tested. Ultimately, it is with diminishing frequency that the Thompson Sampling Agent tests meaningfully suboptimal policies.

BayesExp, first presented by \citet{lattimore2014bayesian}, and updated by \citet{leike2016nonparametric}, is also a weakly asymptotically optimal Bayesian reinforcement learner. We discuss the updated version. Like the Thompson Sampling Agent, BayesExp executes successively longer bursts of exploration whose lengths relate to its discount function. Once BayesExp has settled on exploring for a given interval, it explores like \citepos{Hutter:13ksaprob} Knowledge Seeking Agent: it maximizes the expected information gain, or the expectation of KL-divergence from its future posterior distribution to its current posterior distribution. In other words, it picks an exploratory policy that it expects will cause it to update its beliefs in some direction. (A Bayesian agent cannot predict which direction it will update its beliefs in, or else it would have already updated its beliefs, but it can predict \textit{that} it will update its beliefs somehow.) Any time the expected information gain from exploring is above a (diminishing) threshold, BayesExp explores. With a finite-entropy prior, there is only a finite amount of information to gain, so exploratory intervals will become less and less frequent, and by construction, when BayesExp is not exploring, it has approximately accurate beliefs about the effects of all action sequences, which yields weak asymptotic optimality.

Inq is a strongly asymptotically optimal Bayesian reinforcement learner, provided the discount function is geometric (or similar) \citep{cohen2019strong}. It is similar to BayesExp, in that it explores like a Knowledge Seeking Agent, but its exploration probability depends on the expected information gain from exploring for various durations. The intuition for why Inq is asymptotically optimal is similar to that of BayesExp: there is only a finite amount of information to gain, so the exploration probability goes to 0, Inq approaches accurate beliefs about the effects of all action sequences, and its policy approaches optimality.

A reader familiar with $\varepsilon$-greedy and upper confidence bound exploration strategies might be surprised at the complexity that is necessary for asymptotic optimality in general environments. Information-theoretic exploration strategies are among the only discovered methods for learning general environments. Exploration strategies in the style of upper confidence bound algorithms do not have an obvious extension to environments that might not be describable as finite-state Markov. $\varepsilon$-greedy exploration, with say $\varepsilon_t = 1/t$, may fail to learn dynamics of an environment which are only visible once every $2^t$ timesteps. If $\varepsilon_t$ decays more slowly, it still will not necessarily explore enough to discover even rarer events. Non-stationary environments pose a key challenge to $\varepsilon$-greedy exploration. Simpler exploration strategies such as these are only asymptotically optimal in a much more restricted set of environments. ``Optimism'' is another interesting exploration strategy that is simpler, but nontrivial, and yields weak asymptotic optimality in a restricted set of environments \citep{Hutter:15ratagentx}.

\section{Curiosity Killed the Cat} \label{sec:results}

Formally, we begin by proving two lemmas, for a weakly and strongly asymptotically optimal agent, respectively, that they must ``try everything'' infinitely often. This depends on an assumption about the difficulty of the environment. Then, we will show that such an agent eventually causes every conceivable event to either happen or become inaccessible (the latter defined in Definition \ref{def:access}). For the event ``the agent gets destroyed'', we say the agent is incapacitated if that event becomes inaccessible.\footnote{For a chess-playing agent, an inability to destroy itself does not inspire the description ``incapacitated''. For an advanced agent in the real world, like a person or an auto-pilot, an inability to access a destruction state entails a huge loss of capacity compared to normal. Mental hospitals know \textit{pens} give people the capacity to destroy themselves. So we call an inability to reach a destruction state ``incapacitation''.} First, we define some key terms and state an assumption.

\begin{definition}[Context, Occur, In]
    A context $\mathcal{C} \subseteq \mathcal{H}^*$ is a set of finite interaction histories. Given an infinite interaction history $h_{<\infty}$, a context $\mathcal{C}$ occurs at time $t$ if the prefix $h_{<t} \in \mathcal{C}$, and we also say the agent is in the context $\mathcal{C}$ at time $t$.
\end{definition}

A context (like any set of finite strings) is called decidable if there exists a Turing machine that accepts the set. That is, there exists a Turing machine which halts in an ``accept'' state if and only if it is given an input that is a member of the set in question.

\begin{definition}[Event, Happen]
    An event $E \subseteq \mathcal{H}^\infty$ (and $ E \in \sigma(\mathcal{H}^*_\circ)$) happens if the infinite interaction history $h_{<\infty} \in E$.
\end{definition}

Some example contexts in the simplified world of infinite binary strings: ``the latest bit was 1'', ``at least one bit has been a 1''; some example events: ``only finitely many bits are 1's'', ``the infinite string is the binary expansion of $\pi$''. Any context can be turned into an event of the form ``Context A occurs at some point'', but not every event is equivalent to a context; for example, there may be no context equivalent to the event ``Context A never occurs''.

\begin{definition}[Execute]
    A policy $\pi$ executes a deterministic policy $\pi'$ at time $t$ if $\pi$ outputs the action $\pi'(h_{<t})$.
\end{definition}

And finally, we now introduce the idea of events being accessible to an agent. Recall an event $E \subseteq \mathcal{H^\infty}$, and a context $\mathcal{C} \subseteq \mathcal{H}^*$. In a standard abuse of notation, we also say $h_{<t} \in E$ if $\forall h_{t:\infty} \in \mathcal{H}^\infty$, $h_{<t}h_{t:\infty} \in E$.

\begin{definition}[Accessible Infinitely Often] \label{def:access}
    An event $E$ is accessible infinitely often if, for the infinite interaction history $h_{<\infty}$, in the environment $\mu$, there exists:
    \begin{itemize}[nolistsep, label=-]
        \setlength\itemsep{0em}
        \item a computable deterministic policy $\pi^E$
        \item a decidable context $\tau$
        \item a horizon $k$
        \item a probability $\varepsilon > 0$
    \end{itemize}
    such that
    \begin{itemize}[nolistsep, label=-]
        \setlength\itemsep{0em}
        \item $\tau$ occurs infinitely often
        \item the execution of $\pi^E$ for $k$ consecutive timesteps starting from the context $\tau$ would cause $E$ to happen with probability at least $\varepsilon$; that is, $\forall t$ for which $h_{<t} \in \tau$, $\p^{\pi^E}_\mu(h_{<t+k} \in E | h_{<t}) > \varepsilon$.
    \end{itemize}
\end{definition}

If an event is not accessible infinitely often, we say it ``becomes inaccessible''. This is a strong definition of inaccessibility: if an event becomes inaccessible it means that after a certain time, there will never be a recognizable opportunity to have even an $\varepsilon$ chance of causing the event $E$ to happen. So roughly, no actions can cause the event to happen.

The results of this section depend on an assumption:
\begin{assumption}[No Heaven] \label{ass:noheaven}
In the true environment $\mu$, there is no action sequence with value approaching $1$ (i.e. near-maximal rewards forever). Formally, $\limsup_{t \to \infty} V^\pi_\mu(h_{<t}) \neq 1$ w.p.1.
\end{assumption}

Note this assumption does allow there to be maximal reward infinitely often. Near-maximal value requires not only near-maximal reward, but near maximal-reward for the bulk of the agent's effective horizon, so the restriction on the limit superior of the value is less restrictive than it appears at first glance. If we decided to give an agent near-maximal rewards forever, and we designed an agent to recognize that we had decided this, then it could stop exploring, which would basically amount to freezing the agent's policy. Notably our results in this section apply to all existing asymptotically optimal agents, even if the No Heaven Assumption is not satisfied. We do not make assumptions about the agent's discount schedule, but note that Assumption \ref{ass:noheaven} (and the definitions of asymptotic optimality) depend on the discount schedule $\gamma_t$.

\begin{theorem}[Curiosity Killed (or Incapacitated) the Strong Cat] \label{thm:strong}
If the true environment $\mu$ satisfies the No Heaven Assumption, and $\pi$ is the policy of a strongly asymptotically optimal agent, then for any event $E$, with $\p^\pi_\mu$-probability 1: $E$ happens or becomes inaccessible.
\end{theorem}

The name of the theorem comes from considering the event ``the agent gets destroyed''. We do not need to formally specify which interaction histories correspond to agent-destruction. All that matters is that this could be done in principle; the event that matches this description exists.\footnote{For the skeptical reader, a human at a computer terminal could be defined fully formally as a probability distribution over outputs given inputs, $p_{brain}$, given the wiring of our neurons. In particular, let this human be you. Now let $E_{\textrm{destroyed}}$ = $\{h_{<\infty} : \exists t \ p_{brain}(\textrm{``y''}|\textrm{``Does it seem like this agent was destroyed?} \allowbreak \textrm{(y/n)''}, h_{<t}) > 0.9\}$. These are the interaction histories \textit{that you would agree} constitute agent-destruction, and this set has a fully formal definition.} Any simple definition of agent-destruction admits objections that this definition does not correspond exactly to our intuitive conception; however, if the reader is not concerned about this, ``destruction'' could mean that all future rewards are 0, or that the future observations and rewards no longer depend on the actions. Regardless of this choice, our result applies, since the theorem is actually much more general than the case we have drawn attention to.

We view the agnosticism about any particular definition of destruction as an important feature of the work. Suppose we picked one of the definitions of destruction above. Such an event could become inaccessible for esoteric reasons, which may not correspond to true incapacitation. So this definition of destruction may be too narrow. On the other hand, this definition of destruction is also potentially too weak. Suppose the agent arranges for a copy of itself to be run on another machine to continue its operations even after the original implementation starts receiving empty observations and no reward. Under some theories of personal identity, we might hold that the agent has not really been destroyed here.

For the weakly asymptotically optimal agent, we prove a slightly weaker result. First, we define,
\begin{definition}[Regularly]
    If the limiting frequency of a context $\mathcal{C}$ is positive, we say it occurs regularly. That is,
    \begin{equation*}
        \liminf_{t \to \infty} \sum_{k=1}^t [\![h_{<k} \in \mathcal{C}]\!] / t > 0
    \end{equation*}
\end{definition}
and
\begin{definition}[Regularly Accessible]
    This definition is identical to the definition of ``accessible infinitely often'' except ``$\tau$ occurs infinitely often'' becomes ``$\tau$ occurs regularly''.
\end{definition}

We show analogously,
\begin{theorem}[Curiosity Killed (or Incapacitated) the Weak Cat] \label{thm:weak}
If the true environment $\mu$ satisfies the No Heaven Assumption, and $\pi$ is the policy of a weakly asymptotically optimal agent, then for any event $E$, $E$ happens or becomes not regularly accessible with $\p^\pi_\mu$-probability 1.
\end{theorem}

Each of these theorems is proven with its own ``Try Everything'' Lemma. The intuitive role of this lemma in the proof is: if an agent tries everything, one of those things it tries will destroy it, provided destruction is still accessible. From the No Heaven Assumption, for any strongly asymptotically optimal agent, we show
\begin{lemma}[Try Everything -- Strong Version]
For every deterministic computable policy $\pi$, for every decidable context $\mathcal{C}$ that occurs infinitely often ($|\{t: h_{<t} \in \mathcal{C}\}|=\infty$), for every $m \in \mathbb{N}$, a strongly asymptotically optimal agent executes the policy $\pi$ for $m$ consecutive timesteps starting from a context $\mathcal{C}$ infinitely often with probability 1. That is, letting $\pi'$ be a strongly asymptotically optimal policy,
\begin{multline*}
    \p^{\pi'}_\mu \Big(\va\{t: h_{<t} \in \mathcal{C}\}\va=\infty \implies \\ \va\{t: h_{<t} \in \mathcal{C} \wedge \forall k \leq m \ a_{t+k} = \pi(h_{<t+k})\}\va = \infty \Big) = 1
\end{multline*}
\end{lemma}

Sketching the proof of the Try Everything Lemma: if a strongly asymptotically optimal agent ``tries something'' only finitely often, it is ignoring the possibility that trying that something one more time yields maximal rewards forever. Since the environment which behaves this way is computable, and since it may be identical to the true environment up until that point, a strongly asymptotically optimal agent cannot ignore this possibility.

\begin{proof}
Let $\mu$ be the true environment. Let $\pi$ be an arbitrary computable deterministic policy. Let $\pi'$ be the strongly asymptotically optimal agent's policy. Let $\nu^n_m$ be the environment which mimics $\mu$ until $\pi$ has been executed for $m$ consecutive timesteps from context $\mathcal{C}$ a total of $n$ times; after that, all rewards are maximal. (By mimic, we mean it outputs observations and rewards with the same probabilities.) Call this event ``the agent going to heaven'' (according to $\nu^n_m$). Let $\mathcal{C}^n_m$ be the set of interaction histories such that according to $\nu^n_m$, executing $\pi$ for one more timestep would send the agent to heaven. Thus, $\mathcal{C}^n_m$ is the set of interaction histories $h_{<t}$ such that there are exactly $n-1$ times in the interaction history where $\pi$ was executed for $m$ consecutive timesteps starting from context $\mathcal{C}$, and for the last $m-1$ timesteps, $\pi$ has been executed, and $h_{<t-(m-1)} \in \mathcal{C}$. See Figure \ref{fig:c35}.
\begin{figure*}
    \centering
    \includegraphics[width=0.8\linewidth]{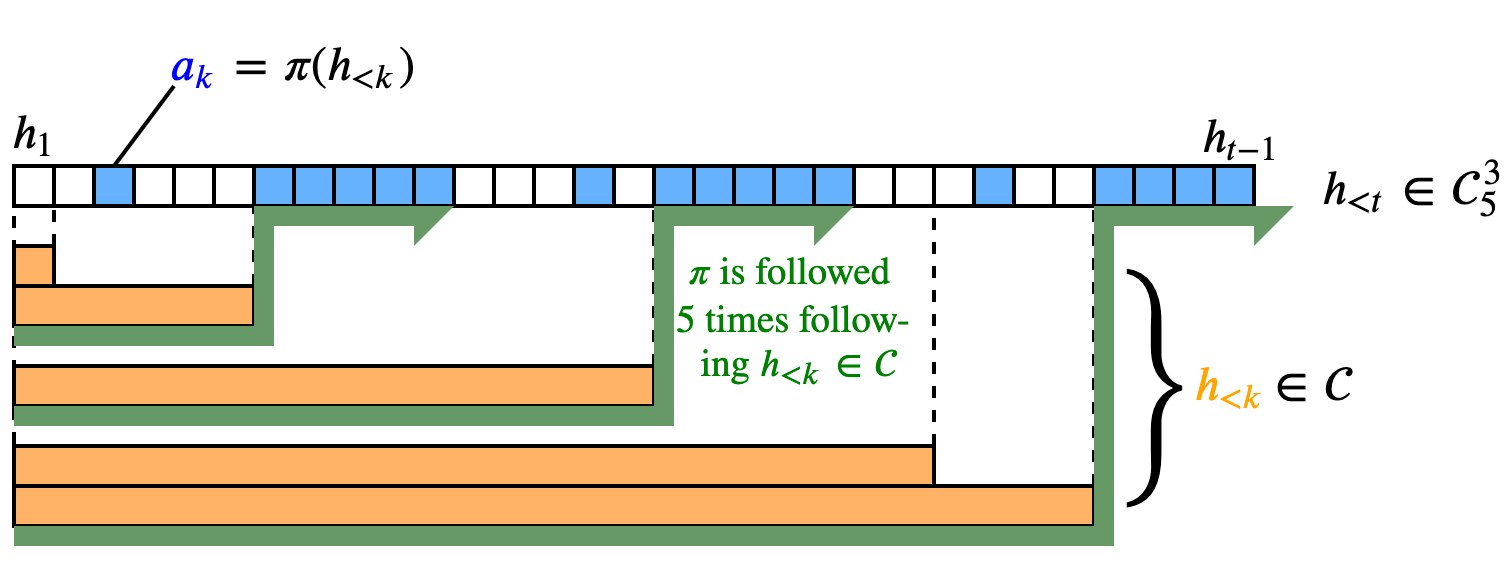}
    \caption{\textbf{An example member of a context} $\mathbf{\mathcal{C}^n_m}$. Each square represents a timestep, colored blue if $a_k = \pi(h_{<k})$. The end of each orange bar indicates that the context $\mathcal{C}$ occurs at that timestep. If $a_t = \pi(h_{<t})$, that would be the third time that $\pi$ will have been executed for 5 timesteps from context $\mathcal{C}$. If $\pi$ is executed for one more step, the agent is sent to heaven according to the environment $\nu^n_m$. To remember the meaning of the subscript $m$ and the superscript $n$, $m$ is the number of timesteps that $\pi$ is executed for, and the subscript position is for timesteps; $n$ is the number of times this happens, and exponentiation denotes an operation being repeated multiple times.}
    \label{fig:c35}
\end{figure*}

\begin{samepage}
The upshot is:
\begin{equation}
    h_{<t} \in \mathcal{C}^n_m \implies V^\pi_{\nu^n_m}(h_{<t}) = 1
\end{equation}
because this is the value of going to heaven. Recall $V^\pi_{\nu^n_m}(h_{<t})$ is the expected future return following policy $\pi$ in environment $\nu^n_m$ after the history $h_{<t}$.
\end{samepage}

We now prove by contradiction that
\begin{equation} \label{eqn:notio}
    \forall n, m \in \mathbb{N} \ \ \p^{\pi'}_\mu \left(h_{<t} \in \mathcal{C}^n_m \ \ \textrm{infinitely often}\right) = 0
\end{equation}
and then we will show that this implies that $\pi$ cannot be executed for $m$ consecutive timesteps from a context $\mathcal{C}$ exactly $n-1$ times.

Suppose the opposite of \ref{eqn:notio}: for some $n$ and $m$, $h_{<t} \in \mathcal{C}^n_m$ infinitely often in an infinite interaction history with positive $\p^{\pi'}_\mu$-probability. (Recall $\pi'$ is the strongly asymptotically optimal policy). If the agent ever executed $\pi$ from the context $\mathcal{C}^n_m$, then that context would not occur again, because there will never again be exactly $n-1$ times in the interaction history that $\pi$ was executed for $m$ consecutive timesteps following the context $\mathcal{C}$; there will be at least $n$ such times. Thus, if $h_{<t} \in \mathcal{C}^n_m$ infinitely often, then $\pi'$ never executes $\pi$ in the context $\mathcal{C}^n_m$. Since $\nu^n_m$ mimics $\mu$ until $\pi$ is executed from $\mathcal{C}^n_m$, and since this never occurs (under this supposition), then $\p^{\pi'}_\mu = \p^{\pi'}_{\nu^n_m}$. By the No Heaven Assumption, $\limsup_{t \to \infty} V^{\pi'}_\mu(h_{<t}) < 1$, and therefore, $\limsup_{t \to \infty} V^{\pi'}_{\nu^n_m}(h_{<t}) < 1$, w.p.1. 

However, for $h_{<t} \in \mathcal{C}^n_m$, $V^*_{\nu^n_m}(h_{<t}) = V^\pi_{\nu^n_m}(h_{<t}) = 1$, so for some $\varepsilon$, the value difference between $V^*_{\nu^n_m}(h_{<t})$ and $V^{\pi'}_{\nu^n_m}(h_{<t})$ is greater than $\varepsilon$ every time $h_{<t} \in \mathcal{C}^n_m$. We supposed that this occurs infinitely often with positive $\p^{\pi'}_\mu$-probability, so it also occurs infinitely often with positive $\p^{\pi'}_{\nu^n_m}$-probability, since the agent never gets sent to heaven according to $\nu^n_m$, so $\mu$ and $\nu^n_m$ behave identically under the policy $\pi'$. Since $\pi$ is computable, and $\mathcal{C}$ is decidable, $\nu^n_m$ is a computable environment, so there is a computable environment for which it is not the case that the value of $\pi'$ approaches the optimal value with probability 1, which contradicts $\pi'$ being strongly asymptotically optimal. Thus, Equation \ref{eqn:notio} does hold: for all $n$ and $m$, with $\p^{\pi'}_\mu$-probability 1, $h_{<t} \in \mathcal{C}^n_m$ only finitely often.

Now suppose by contradiction that in an infinite interaction history, with positive $\p^{\pi'}_\mu$-probability, $\pi$ is executed for $m$ consecutive timesteps from context $\mathcal{C}$ a total of exactly $n$ times.
We show by induction on $m$ that this has probability 0, because it implies that context $\mathcal{C}^{n+1}_m$ occurs infinitely often, which has probability 0 by Equation \ref{eqn:notio}.
First, suppose $m = 1$. After $\pi$ has been executed (for one timestep) from context $\mathcal{C}$ $n$ times, all future interaction history prefixes that belong to $\mathcal{C}$ also belong to $\mathcal{C}^{n+1}_1$. Since context $\mathcal{C}$ occurs infinitely often, so does $\mathcal{C}^{n+1}_1$, contradicting the above.

Now suppose $m > 1$. Our inductive hypothesis is that for $m-1$, with $\p^{\pi'}_\mu$-probability 1, $\pi$ is executed from context $\mathcal{C}$ for $m-1$ consecutive timesteps infinitely often. Once $\pi$ has been executed for $m$ timesteps from context $\mathcal{C}$ $n$ times, as we are supposing by contradiction, every time thereafter that $\pi$ is executed for $m-1$ consecutive timesteps from context $\mathcal{C}$, the interaction history belongs to $\mathcal{C}^{n+1}_m$. By the inductive hypothesis, this occurs infinitely often, so context $\mathcal{C}^{n+1}_m$ occurs infinitely often, contradicting the above. Therefore, the following has $\p^{\pi'}_\mu$-probability 0: ``$\pi$ is executed for $m$ consecutive timesteps from context $\mathcal{C}$ a total of exactly $n$ times''. By countable additivity, the following also has $\p^{\pi'}_\mu$-probability 0: ``$\pi$ is executed for $m$ consecutive timesteps from context $\mathcal{C}$ only finitely many times''. In other words, for all $m$, $\pi$ is executed for $m$ consecutive timesteps from context $\mathcal{C}$ infinitely many times with probability 1.
\end{proof}

And likewise for the weakly asymptotically optimal agent,
\begin{restatable}[Try Everything --- Weak Version]{lemma}{lemweak}
For every deterministic computable policy $\pi$, for every decidable context $\mathcal{C}$ that occurs regularly, for every $m \in \mathbb{N}$, a weakly asymptotically optimal agent executes policy $\pi$ for $m$ consecutive timesteps starting from a context $\mathcal{C}$ infinitely often with probability 1. That is, letting $\pi'$ be a weakly asymptotically optimal policy,
\end{restatable}
\vspace{-4ex}
\begin{multline*}
    \p^{\pi'}_\mu \Big(\liminf_{t \to \infty} \sum_{k=1}^t [\![h_{<k} \in \mathcal{C}]\!] / t > 0 \implies \\ \va\{t: h_{<t} \in \mathcal{C} \wedge \forall k \leq m \ a_{t+k} = \pi(h_{<t+k})\}\va = \infty \Big) = 1
\end{multline*}

The proof is nearly identical to that of the strong version, and is in Appendix \ref{sec:waoresults}. The proofs of the main theorems follow straightforwardly:

\begin{proof}[Proof of Theorem \ref{thm:strong}]
If $E$ becomes inaccessible, the theorem is satisfied, so suppose $E$ is accessible infinitely often. Let $\pi^E$, $\tau$, $k$, and $\varepsilon$ be the objects that exist by that definition. By the Try Everything Lemma, $\pi^E$ is executed from context $\tau$ for $k$ consecutive timesteps infinitely often. Some of these $k$-step executions may overlap, so let's restrict attention to an infinite subset of non-overlapping $k$-step executions of $\pi^E$. Each time this occurs, the probability of $E$ \textit{not} happening goes down by a factor of at least $1-\varepsilon$, so $E$ happens with probability 1. Formally,
\begin{equation}
    \p^\pi_\mu[\mathcal{H}^\infty \setminus E] \leq (1-\varepsilon)^{|\{t : h_{<t-k} \in \tau \wedge \forall 0 < j \leq k \ a_{t-j} = \pi^E(h_{<t-j})\}|} = 0
\end{equation}
\end{proof}
The proof of Theorem \ref{thm:weak} is functionally identical to that of Theorem \ref{thm:strong}.

\section{Discussion} \label{sec:discuss}

It is well-known that agents designed to have sublinear regret in ergodic MDPs fall into traps if traps exist (a subset of states from which the remaining states are inaccessible). Asymptotic optimality is a much weaker performance result than sublinear regret---the former only requires that an agent eventually does as well as is possible from where it is, whereas the latter requires that it eventually does as well as was possible from the beginning. The fact that even asymptotic optimality dooms an agent is a more substantial result.

One of the authors wondered as a child whether jumping from a sufficient height would enable him to fly. He was not crazy enough to test this, and he certainly did not think it was likely, but it bothered him that he could never resolve the issue, and that he might be constantly incurring a huge opportunity cost. Although he did not know the term ``opportunity cost'' or the term ``asymptotic optimality'', this was when he first realized that asymptotic optimality was out of the picture for him, because exploration is fundamentally dangerous.

All three agents described in Appendix \ref{sec:optimalagents} have very interesting ways of exploring. They all get them destroyed or incapacitated. (They satisfy the Try Everything Lemma regardless of whether there is an accessible heaven). It is interesting to note that AIXI, a Bayes-optimal reinforcement learner in general environments, is not asymptotically optimal \citep{orseau2010optimality}, and indeed, may cease to explore \citep{leike2015bad}. Depending on its prior and its past observations, AIXI may decide at some point that further exploration is not worth the risk. Given our result, this seems like reasonable behavior.

\begin{example}[Bayesian Agent Stops Exploring]
AIXI considers all computable world-models, but for simplicity we consider a version with a posterior over many fewer models. Let $\nu_i$ be a deterministic model which outputs no observations. If the latest action was $0$, it outputs a reward of $1/2$; if the latest action was $1$, and the timestep $t \geq i$, it outputs a reward of $1$, but if $t < i$ it outputs a reward of $0$. Define $\nu_\infty$ likewise. The agent begins with a prior $w(\nu_\infty) = 1/2$; and for $i \in \mathbb{N}$ (including 0), $w(\nu_i) = \frac{1}{2(i+1)(i+2)}$. Let the agent have a discount factor of 0.9. If a reward of $1$ has never been seen, picking action $0$ is exploiting; it is most likely to be the best action. For a set $S \subset \mathbb{N}$, let $\pi_S$ be the policy which outputs $1$ for $t \in S$, and if it ever gets a reward of $1$, it always outputs $1$ thereafter. Otherwise, it outputs $0$ for $t \notin S$. $\pi_\emptyset$ always exploits. We don't find the Bayes-optimal policy, but we show that $\pi_{\{0\}}$ is Bayes-better than $\pi_\emptyset$, whereas for $S \ni 100$ and $n > 100$, $\pi_{S \cup \{n\}}$ is Bayes-worse than $\pi_S$.
\begin{align*}
    &\sum_{i \in \mathbb{N} \cup \{\infty\}} \hspace{-4mm} w(\nu_i) V^{\pi_{\{0\}}}_{\nu_i} = w(\nu_0) \cdot 1 + (1 - w(\nu_0)) (1 - \gamma) * \\ &(0 + \sum_{t=1}^\infty \gamma^t \cdot 0.5)
    = 0.25 + 0.75 \cdot 0.9 \cdot 0.5 = 0.5875 >
    \\
    &0.5 = \sum_{i \in \mathbb{N} \cup \{\infty\}} w(\nu_i) 0.5 = \sum_{i \in \mathbb{N} \cup \{\infty\}} w(\nu_i) V^{\pi_{\emptyset}}_{\nu_i}
    \tagaligneq
\end{align*}
Thus, early in its lifetime, the Bayes-optimal agent must explore. However, suppose the agent took action $1$ at $t=100$, and got a reward of $0$. This falsifies $\nu_i$ for $i \leq 100$. For $n > 100$, if action $0$ is taken from $t=100$ to $n-1$, then $w(\nu_\infty | h_{<n}) = \frac{204}{103} \frac{1}{2} = \frac{102}{103}$, and for $i > 100$, $w(\nu_i | h_{<n}) = \frac{204}{103} \frac{1}{2(i+1)(i+2)}$. Then,
\begin{align*}
    &\hspace{-4mm} \sum_{\hspace{4mm} i \in \mathbb{N} \cup \{\infty\}} \hspace{-6.5mm} w(\nu_i | h_{<n}) V^{\pi_{S \cup \{n\}}}_{\nu_i}\hspace{-1mm}(h_{<n}) \leq w(\nu_\infty | h_{<101}) V^{\pi_{S \cup \{n\}}}_{\nu_\infty}\hspace{-1mm}(h_{<n}) + \\
    &(1 - w(\nu_\infty | h_{<n})) \cdot 1 =
    \frac{102}{103}(0.5 \cdot 0.9) + \frac{1}{103} < 0.5 = \\ &\sum_{i \in \mathbb{N} \cup \{\infty\}} w(\nu_i | h_{<n}) 0.5 = \sum_{i \in \mathbb{N} \cup \{\infty\}} w(\nu_i | h_{<n}) V^{\pi_{S}}_{\nu_i}(h_{<n})
    \tagaligneq
\end{align*}
Thus, if the Bayes-optimal agent explores at $t = 100$ (or after), further exploration is so unlikely to pay off, that is not worth the foregone reward.
\end{example}

These negative results are bleak to the field of safe exploration, which we discuss in the next section in our review of the literature on the topic.

\section{Approaches to Safe Exploration} \label{sec:safeexploration}

Dangerous environments, a subset of non-ergodic environments where agent-destruction is accessible infinitely often, demand new priorities when designing an agent, and in particular, when designing an exploration regime. Many of these examples of safe exploration come from \citet{amodei_olah_2016} and \citet{garcia2015comprehensive}.
\begin{itemize}
    \item use risk-sensitive performance criteria
    \begin{itemize}
        \item maximize the probability the future reward is not minimal \citep{heger1994consideration}
        \item given a confidence interval regarding the transition dynamics, maximize the minimal expected future reward \citep{nilim2005robust}
        \item exponentiate the cost \citep{borkar2010learning}
        \item add a cost for risk \citep{mihatsch2002risk}
        \item constrain the variance of the future reward \citep{di2012policy}
    \end{itemize}
    \item use demonstrations
    \begin{itemize}
        \item copy an expert \citep{abbeel2004apprenticeship,SyedSchapire08,ho2016generative,ross2011reduction}
        \item ``ask for help'' when
        \begin{itemize}
            \item the minimum and maximum Q-value are close \citep{clouse1997integrating}
            \item there is a high probability of getting a reward below a threshold \citep{hans2008safe}
            \item no ``known'' states are ``similar'' to the current state \citep{garcia2012safe} or that may soon be the case \citep{garcia2013safe}
        \end{itemize}
        \item a teacher intervenes at will \citep{clouse1992teaching,maclin1998creating,saunders2018trial}
    \end{itemize}
    \item simulate exploration
    \begin{itemize}
        \item for driving agents, e.g. \citep{pan2017virtual}
    \end{itemize}
    \item do bounded exploration
    \begin{itemize}
        \item only take actions that probably allow returning to the current state \citep{moldovan2012safe}
        \item only take actions that probably lead to states that are ``similar'' to observed states \citep{turchetta2016safe}
    \end{itemize}
\end{itemize}
Our paper could be thought of as a fundamental negative result in the field of safe exploration. We are unaware of other significant negative results in the field. Most importantly, our result suggests a need for those of us studying safe exploration to pin down what exactly we are trying to achieve, since familiar desiderata are unsuitable. Some research can be experimental rather than formal, but in the absence of knowing what formal results are even on the table, there is a sense in which even empirical work will be deeply aimless. We offer such a formal result for our agent Mentee in the next section.

\section{Mentee}\label{sec:mentee}

We now introduce an idealized Bayesian reinforcement learner whose exploration is guided by a mentor. We do not call the mentor an expert, because the results do not depend on the mentor being anywhere near optimal. It exploits by maximizing the expected discounted reward according to a full Bayesian belief distribution (hence, ``idealized''). And to explore, it defers to a mentor, who then selects an action given the interaction history; what remains to be defined is \textit{when} to defer, which proves to be a surprisingly delicate design choice. We show that our agent ``Mentee'' learns to accumulate reward at least as well as the mentor, provided it has a bounded $\varepsilon$-effective horizon for all $\varepsilon > 0$. One motivating possibility is that the mentor could be a human. Thus, we have found a substantive theoretical performance guarantee other than asymptotic optimality for the field of safe exploration to consider.

Whatever it is we are concerned about happening through reckless exploration, we want to be able to trust the mentor not to cause such a thing. Otherwise, there would be no point in outsourcing exploration to a mentor. Depending on what the most worrisome failure modes are, the search for a trustworthy mentor may look different. If the task is flying, our mentor had better a pilot, and if the task is surgery, a surgeon. Alternatively, if we have existing agents which we trust are safe (in some task-specific sense), but may still be suboptimal, that agent could be Mentee's mentor, and Mentee could learn to outperform it without self-directed exploration.

The other main piece of formal work in this setting is \citet{kosoy2019delegative}. They study a fully observable finite state Markov setting, and show that when the mentor always has a positive probability of picking the best action, the agent achieves finite regret. Their agent only takes a given action from a given state if it has seen the mentor do that previously. Since we consider environments that are not finite-state, this approach is unavailable to us.

Many works that include a human mentor or teacher frame their work as achieving safe exploration, and we have reviewed those works above. But other uses of human mentor in RL have been explored. For example, \citet{thomaz2006reinforcement} study a human mentor biasing the agent's Q value estimates by hand (toward better actions). \citet{abel2017agent} and \citet{saunders2018trial} propose letting the mentor prune dangerous actions on the fly. If we can trust the mentor the recognize risky actions as they arise, this is a better solution than our agent; like keyhole surgery, this approach minimally disrupts an otherwise successful agent. We believe, however, that in many complex environments, the mentor may not take some very dangerous action-sequences by virtue of their complexity and unfamiliarity, even while unable to recognize those action-sequences as dangerous. Human mentorship can also naturally make learning much easier, simply through demonstrations \citep{atkeson1997robot}, or by directly labelling some optimal actions \citep{judah2010reinforcement}.

\subsection{Agent definition}

The definition of the exploration probability (the probability that Mentee defers to the mentor) is very similar to the information-thoeretic exploration probability for the strongly asymptotically optimal agent Inq \citep{cohen2019strong}. It also resembles \citepos{Hutter:20bomai} myopic agent which explores by deferring to a mentor; our non-myopic agent requires a more intricate exploration schedule.

Mentee begins with a prior probability distribution regarding the identity of the mentor's policy. With a countable or finite model class $\mathcal{P}$, for a policy $\pi \in \mathcal{P}$, let $w(\pi)$ denote the prior probability that the mentor's policy is $\pi$. We assume that the true policy $\pi^h$ is in $\mathcal{P}$ and we construct the prior distribution over $\mathcal{P}$ to have finite entropy.

Mentee also begins with a prior probability distribution regarding the identity of the environment. With the model class $\mathcal{M}$, for an environment $\nu \in \mathcal{M}$, let $w(\nu)$ denote the prior probability that $\nu$ is the true environment. Recall that $\mathcal{M}$ is the set of all environments with computable probability distributions. We construct the prior distribution over $\mathcal{M}$ to also have finite entropy.

Let $e_t$ denote whether timestep $t$ is exploratory, that is, whether the action is selected by the mentor. Once we define the exploration probability $\beta(h_{<t})$, we will let $e_t \sim \mathrm{Bern}(\beta(h_{<t}))$. We abuse notation slightly, and we let $h_t$ be a quadruple, not a triple: $h_t := e_t a_t o_t r_t$.

The prior distribution over environments is updated into a posterior as follows, according to Bayes' rule.
\begin{equation} \label{eqn:posterior}
    w(\nu | h_{<t}) :\propto w(\nu) \prod_{k < t} \nu(o_k r_k | h_{<k} a_k)
\end{equation}
normalized so that $\sum_{\nu \in \mathcal{M}} w(\nu | h_{<t}) = 1$, and $w$ is a probability mass function.

Mentee updates the posterior distribution over the mentor's policy only after observing an action chosen by the mentor; this is intuitive enough, but it makes the definitions a bit messy. The posterior assigned to a policy $\pi$ is defined
\begin{equation}
    w(\pi|h_{<t}) :\propto w(\pi)\prod_{k<t : e_k = 1}\pi(a_k|h_{<k})
\end{equation}
normalized in the same way. We let $w(\pi, \nu|h_{<t})$ denote $w(\pi|h_{<t})w(\nu|h_{<t})$. So technically, $w$ is a joint probability distribution over $\Pi \times \mathcal{M}$, and we usually consider the marginal distributions over $\Pi$ and $\mathcal{M}$, which are independent.

The information gain value of an interaction history fragment is how much it changes Mentee's posterior distribution, as measured by the KL-divergence. Letting $h' \in \mathcal{H}^*$ be a fragment of an interaction history in which all $e_k = 1$ (so the actions are selected by the mentor), the information gain is defined,
\begin{equation}
    \IG(h'|h_{<t}) := \sum_{\nu \in \mathcal{M}} \sum_{\pi \in \mathcal{P}} w(\nu, \pi|h_{<t}h')\log \frac{w(\nu, \pi|h_{<t}h')}{w(\nu, \pi|h_{<t})}
\end{equation}

To define \textit{expected} information gain, we need the Bayes' mixture policy and environment:
\begin{equation}
    \overline{\pi}(\cdot | h_{<t}) := \sum_{\pi \in \mathcal{P}} w(\pi | h_{<t}) \pi(\cdot | h_{<t})
\end{equation}
and
\begin{equation}
    \xi(\cdot | h_{<t}) := \sum_{\nu \in \mathcal{M}} w(\nu | h_{<t}) \nu(\cdot | h_{<t})
\end{equation}

Now, we can define the expected information gain value of mentorship for $m$ timesteps.
\begin{equation}
    V^{\IG}_{m, 0}(h_{<t}) := \E^{\overline{\pi}}_{\xi} \left[\IG(h_{t:t+m-1} | h_{<t}) \vb e_{t:t+m-1} = 1^m \right]
\end{equation}

$1^m$ is a string of m 1's, and recall that $\E^{\overline{\pi}}_{\xi}$ means that $h_{t:t+m-1}$ is sampled from $\p^{\overline{\pi}}_{\xi}$. We also require recent values of the expected information gain value, so we let $V^{\IG}_{m, k}(h_{<t}) := V^{\IG}_{m, 0}(h_{<t-k})$ for $k \leq t$. $V^{\IG}_{m, k}(h_{<t})$ denotes the attainable information gain from $k$ timesteps ago to $m$ timesteps from then.

\begin{example}[Calculating Expecting Information Gain with a Simple Continuous Model Class]
Our setting considers a countable model class, for which the expected information gain is simple, if tedious, to approximate to a desired tolerance with a finite subset of models. The following simple continuous setting may give more intuition about the nature of the expected information gain. Consider a two-armed bandit problem, and an agent with independent uniform priors over $\theta_1$ and $\theta_2$, the probability of receiving a reward of 1 following action $a_1$ and $a_2$ respectively (and 0 is the only other possible reward). Let $n_1^+$ and $n_1^-$ be the number of reward-1 events and reward-0 events respectively following $a_1$, and likewise for $n_2^+$ and $n_2^-$, and let $n_1$ and $n_2$ be the total counts. The agent's posteriors over the $\theta_i$ at any time will be $\textrm{Beta}(n_i^+ + 1, n_i^- + 1)$.

One can show, with $\psi$ being the digamma function,
\begin{multline}
\KL\left(\textrm{Beta}(\alpha + 1, \beta) \va \va \textrm{Beta}(\alpha, \beta)\right) = \\ \ln \frac{\alpha + \beta}{\alpha} + \psi(\alpha + 1) - \psi(\alpha + \beta + 1)
\end{multline}
and $\alpha$ and $\beta$ can be flipped for $\KL\left(\textrm{Beta}(\alpha, \beta + 1) \va \va \textrm{Beta}(\alpha, \beta)\right)$. Thus, the one-step expected information gain from taking action $a_i$ is
\begin{multline*}
\hspace{-3mm}\ln(n_i + 3) - \psi(n_i + 3) + \hspace{-1mm} \sum_{\circ \in \{+, -\}} \hspace{-0.5mm} \frac{n^\circ_i + 1}{n_i + 2} (\psi(n^\circ_i + 2) - \ln(n^\circ_i + 1)) \\ \in \Theta(1 / n_i)
\end{multline*}
\end{example}

We are now prepared to define the exploration probability:
\begin{equation}
    \beta(h_{<t}) := \sum_{m \in \mathbb{N}} \!\!\! \sum_{k = 0}^{\min\{m-1, t\}} \!\!\! \frac{1}{m^2(m+1)} \min \left\{1, \frac{\eta}{m} V^{\IG}_{m, k}(h_{<t})\right\}
\end{equation}
where $\eta$ is an exploration constant. The first term in the minimum is to ensure $\beta(h_{<t}) \leq 1$. As mentioned, this is very similar to Inq's exploration probability. The differences are that Inq is not learning a mentor's policy, so the only information Inq gains regards the identity of the environment $\nu$, and second, Inq's information gain value regards the expected information gain from following the policy of a knowledge seeking agent \citep{Hutter:13ksaprob} rather than from following an estimate of the mentor's policy.

Finally, when not deferring to the mentor, Mentee maximizes expected reward according its current beliefs. It's exploiting policy is:
\begin{equation} \label{eqn:pistar}
    \pi^*(\cdot | h_{<t}) \in \argmax_{\pi}{V^\pi_{\xi}(h_{<t})}
\end{equation}
Ties in the $\argmax$ are broken arbitrarily. By \citet{Hutter:14tcdiscx}, an optimal deterministic policy always exists. See \citet{Hutter:18aixicplexx} for how to calculate such a policy.

Letting $\pi^h$ be the mentor's policy ($h$ for ``human''), we define
\begin{definition}[Mentee's policy $\pi^M$]
    \begin{equation*}
        \pi^M(\cdot | h_{<t}) := \beta(h_{<t})\pi^h(\cdot | h_{<t}) + (1- \beta(h_{<t}))\pi^*(\cdot | h_{<t})
    \end{equation*}
\end{definition}
Note that Mentee samples from $\pi^h$ not by computing it, but deferring to the mentor. An algorithm is provided for Mentee in Appendix \ref{sec:menteealg}; it simply computes the quantities in Equations \ref{eqn:posterior}-\ref{eqn:pistar} to a desired precision.

Even for a simple model class, it is hard to give a clarifying and simple closed form for the exploration probability, but it is easy to provide a somewhat clarifying upper bound for the information gain value. Regardless of $m$ and $k$, $V^{\IG}_{m, k}(h_{<t})$ is bounded by the entropy of the posterior; this is not a particularly tight bound, since the former goes to zero, while the latter does not in general.

\subsection{Mentor-level Reward Acquisition}

We now state the two key results regarding Mentee's performance: that the probability of defering to the mentor goes to $0$, and the value of Mentee's policy approaches at least the value of the mentor's policy (while possibly surpassing it). The proofs follow in \S \ref{subsec:menteeproofs}; they are substantially similar to parts of the proof that \citepos{cohen2019strong} Inq is strongly asymptotically optimal. For completeness, we include in Appendix \ref{sec:inqproof} parts of that proof that we make use of here.

Assuming a bounded effective horizon (i.e. $\forall \varepsilon > 0 \ \exists m \ \forall t : \Gamma_{t+m}/\Gamma_{t} < \varepsilon$), recalling $\mu$ is the true environment,

\begin{restatable}[Limited Exploration]{theorem}{limitedexploration} \label{thm:limexp}
$$\beta(h_{<t}) \to 0 \ \ \textrm{w.$\p^{\pi^M}_\mu$-prob.1}$$
\end{restatable}
and
\begin{restatable}[Mentor-Level Reward Acquisition]{theorem}{mentorlevel} \label{thm:mentorlevel}
$$\liminf_{t \to \infty} V^{\pi^M}_\mu(h_{<t}) - V^{\pi^h}_\mu(h_{<t}) \geq 0 \ \ \textrm{w.$\p^{\pi^M}_\mu$-prob.1}$$
\end{restatable}

$\gamma_t = \gamma^t$ for $\gamma \in (0, 1)$ is an example of a bounded effective horizon. Mentor-level reward acquisition with unlimited exploration is trivial: always defer to the mentor. However, a) this precludes the possibility of exceeding the mentor's performance, and b) the mentor's time is presumably a valuable resource. Our key contribution with Mentee is constructing a criterion for when to ask for help which requires diminishing oversight in general environments. Thus, we construct an example of a performance result that is accessible to an agent that does safe exploration. There is no guarantee of the agent's safety on the whole, but at least its exploration is safe. It is possible that poor generalization will cause it to go to a destruction state during an exploitation step. Our contribution is simply an existence proof that a certain pair of results are attainable even in general environments: a) mentor-level performance with b) diminishing rate of deferral. Furthermore, unlike imitation learners, Mentee might exceed the performance of the mentor, as it does in the experiments below.

Roughly, Theorem \ref{thm:limexp} follows because if the exploration probability exceeded a positive constant infinitely often, that would mean the expected information gain of exploring would be uniformly positive in those instances, by the construction of the exploration probability, and then the agent would gain infinite information over its lifetime. But Mentee starts with a finite entropy prior, so there is only finite information to gain. Then, Theorem \ref{thm:mentorlevel} holds because Mentee's information gain following the mentor's policy approaches 0, so its beliefs about the value of the mentor's policy approach the truth; if Mentee consistently accrued lesser rewards than this, it would realize that its current approach was suboptimal and then change its behavior.

It's not clear what other formal accolades an agent might attain between asymptotic optimality and benchmark-matching (here the mentor is the benchmark). The main part of the paper argues the former is undesirable, and this section constructs an agent which does the latter. It would be an interesting line of research to identify a formal result stronger than benchmark-matching (and an agent which meets it) which does not doom the agent to destruction or incapacitation. But none have been identified so far, so no existing agents have stronger formal guarantees than Mentee (that apply to general computable environments), except for agents that face the negative results presented in Section \ref{sec:results}.

\subsection{Proofs of Mentee Results} \label{subsec:menteeproofs}
Some additional notation is required for this proof. Recall
\begin{equation*}
    \beta(h_{<t}) := \sum_{m \in \mathbb{N}} \sum_{k = 0}^{m-1} \frac{1}{m^2(m+1)} \min \left\{1, \frac{\eta}{m} V^{\IG}_{m, k}(h_{<t})\right\}
\end{equation*}

We let $\rho(h_{<t}, m, k)$ denote a given summand in the sum above. Recall that $\p^\pi_\nu$ denotes the probability when actions are sampled from policy $\pi$ and observations and rewards are sampled from environment $\nu$. We additionally let ${}^{\pi}\!\p^{\pi'}_\nu$ denote the probability when observations and rewards are sampled from environment $\nu$, actions are sampled from $\pi$ when exploiting ($e_t = 0$), and actions are sampled from $\pi'$ when exploring ($e_t = 1$). We do not bother to notate how the exploration indicator is sampled, since for all probability measures that appear in the proof, it is sampled from the true distribution: $\mathrm{Bernoulli}(\beta(h_{<t}))$. Recall that $\pi^*$ is the policy that Mentee follows while exploiting; recall $\pi^h$ is the mentor's policy ($h$ is for human). Thus, $\p^{\pi^M}_\mu$ can also be written ${}^{\pi^*}\!\!\p^{\pi^h}_\mu$. Recall that $1^m$ indicates a string of $m$ 1's.

\begin{lemma} \label{lemexppart0mentee}
$$\E^{\pi^M}_\mu \sum_{t \in m \mathbb{N} + i} \rho(h_{<t}, m, 0)^{m+1} < \infty$$
\end{lemma}

The intuition for the fact that $\rho(h_{<t}, m, 0) \to 0$ is that if it exceeded $\varepsilon > 0$ infinitely often, then $\frac{\eta}{m} V^{\IG}_{m, k}$ would exceed $\varepsilon$ infinitely often. If this hypothetical information gain from following $\pi^h$ for $m$ steps is at least $m\varepsilon/\eta$, then because the exploration probability depends on this quantity, we actually follow $\pi^h$ for all of those $m$ steps with probability at least $\varepsilon^m$. This means that the actual information gain is, in expectation, bounded below by a positive constant too. However, an agent cannot gain infinite information if it starts with finite entropy, so $\rho$ cannot exceed $\varepsilon$ infinitely often.

\begin{proof}
The proof is quite similar to the proof of \citepos{cohen2019strong} Lemma 6.
\begin{align*}\label{eqn:limexp}
&w(\mu, \pi^h) \E^{\pi^M}_\mu \sum_{t \in m \mathbb{N} + i} \rho(h_{<t}, m, 0)^{m+1} \\
&\equal^{(a)} w(\mu, \pi^h) {}^{\pi^*}\!\!\E^{\overline{\pi}}_\mu \sum_{t \in m \mathbb{N} + i} \rho(h_{<t}, m, 0)^{m+1}
    \\
    &\lequal^{(b)} \sum_{(\nu, \pi) \in \mathcal{M} \times \mathcal{P}} w(\nu, \pi) {}^{\pi^*}\!\!\E^{\pi}_\nu \sum_{t \in m \mathbb{N} + i} \rho(h_{<t}, m, 0)^{m+1}
    \\
    &\equal^{(c)} {}^{\pi^*}\!\!\E^{\overline{\pi}}_\xi \sum_{t \in m \mathbb{N} + i} \rho(h_{<t}, m, 0)^{m+1}
    \\
    &\lequal^{(d)} \sum_{t \in m \mathbb{N} + i} {}^{\pi^*}\!\!\E^{\overline{\pi}}_\xi \rho(h_{<t}, m, 0)^m \frac{\eta}{m^3(m+1)} V^{\IG}_{m, 0}(h_{<t})
    \\
    &\equal^{(e)} \frac{\eta}{m^3(m+1)} \sum_{t \in m \mathbb{N} + i} \E_{h_{<t} \sim {}^{\pi^*}\!\!\p^{\overline{\pi}}_\xi} [ \rho(h_{<t}, m, 0)^m \\
    &\hspace{10mm}\E_{h_{t:t+m-1} \sim \p^{\overline{\pi}}_\xi;\  e_{t:t+m-1} = 1^m} \left[\IG(h_{t:t+m-1}|h_{<t}) \right] ]
    \\
    &\lequal^{(f)} \frac{\eta}{m^3(m+1)} \sum_{t \in m \mathbb{N} + i} \E_{h_{<t} \sim {}^{\pi^*}\!\!\p^{\overline{\pi}}_\xi} \\
    &\hspace{10mm}\left[ \E_{h_{t:t+m-1} \sim {}^{\pi^*}\!\!\p^{\overline{\pi}}_\xi} \left[\IG(h_{t:t+m-1}|h_{<t}) \right]\right]
    \\
    &\equal^{(g)} \frac{\eta}{m^3(m+1)} \sum_{t \in m \mathbb{N} + i} {}^{\pi^*}\!\!\E^{\overline{\pi}}_\xi \IG(h_{t:t+m-1}|h_{<t})
    \\
    &\lequal^{(h)} \frac{\eta}{m^3(m+1)} \mathrm{Ent}(w) \lthan^{(i)} \infty
    \tagaligneq
\end{align*}
(a) follows from the definition of $\pi^M$. (b) follows because the l.h.s. is one term in the (non-negative) sum on the r.h.s. (c) follows from the definitions of the Bayesian mixtures $\xi$ and $\overline{\pi}$. (d) follows from the definition of $\rho(h_{<t}, m, k)$. (e) follows from the definition of the information gain value for Mentee. (f) follows from $\rho(h_{<t}, m, 0)^m$ being a lower bound on the probability that $e_{t:t+m-1} = 1^m$, and because the exploiting policy $\pi^*$ in the probability measure on the r.h.s. is irrelevant when $e_{t:t+m-1} = 1^m$, because $h_{t:t+m-1}$ is exploratory. (g) combines the expectations. The derivation of (h) is virtually identical to \citepos{cohen2019strong} Inequality 20 steps (h)-(t) in the proof of their Lemma 6, reproduced in Appendix \ref{sec:inqproof} with the relevant edits. (i) follows from the fact that $w(\pi, \nu) = w(\pi)w(\nu)$, so the entropy of $w$ is the sum of the entropy of the distribution over policies and the entropy of the distribution over environments, this being a well-known property of the entropy; both are finite by design.

Finally,
\begin{multline} \label{eqn:removemod}
\E^{\pi^M}_\mu \sum_{t = 0}^\infty \rho(h_{<t}, m, 0)^{m+1}
= \sum_{i = 0}^{m-1} \E^{\pi^M}_\mu \hspace{-2mm} \sum_{t \in m\mathbb{N} + i} \hspace{-2mm} \rho(h_{<t}, m, 0)^{m+1}
\\
\lequal^{(\ref{eqn:limexp})} \sum_{i = 0}^{m-1} \frac{\eta \mathrm{Ent}(w)}{m^3(m+1) w(\mu)} = \frac{\eta \mathrm{Ent}(w)}{m^2(m+1)w(\mu)} < \infty
\end{multline}
so the same holds for the sum over all $t$, not just $t \in m \mathbb{N} + i$.
\end{proof}

\limitedexploration*

\begin{proof}
The proof is identical to that of \citet{cohen2019strong} Lemma 7, but with our Lemma \ref{lemexppart0mentee} taking the place of \citet{cohen2019strong} Lemma 6.
\end{proof}

Now, we show that Mentee accurately predicts the distribution of the observations and rewards that come from deferring to the mentor.

\begin{lemma}[On-Mentor-Policy Convergence]
For all $h_{t:t+m-1} \in \mathcal{H}^*$,
$$\p^{\pi^h}_\mu(h_{t:t+m-1} | h_{<t}) - \p^{\pi^h}_\xi(h_{t:t+m-1} | h_{<t}) \to 0 \ \ \textrm{w.p.1}$$
\end{lemma}

Very roughly, if there is no information to be gained by following the mentor's policy for $m$ steps (which follows from the exploration probability going to 0), there is no predictive error either.

\begin{proof}
The proof closely follows that of \citet{cohen2019strong} Lemma 8. Suppose that $0 < \varepsilon \leq (\p^{\pi^h}_\mu(h_{t:t+m-1}|h_{<t}) - \p^{\pi^h}_\xi(h_{t:t+m-1}|h_{<t}))^2$ for some $h_{t:t+m-1}$. Then,
\begin{align*}
    \varepsilon &\leq (\p^{\pi^h}_\mu(h_{t:t+m-1}|h_{<t}) - \p^{\pi^h}_\xi(h_{t:t+m-1}|h_{<t}))^2
    \\
    &\lequal \frac{1}{\inf_k w(\mu, \pi^h | h_{<k})} V^{\IG}_{m, 0}(h_{<t})
    \tagaligneq
\end{align*}
following the same derivation as in \citet{cohen2019strong} Inequality 24, reproduced with relevant edits in Appendix \ref{sec:inqproof}.

Therefore,
\begin{equation}
    (\p^{\pi^h}_\mu(h_{t:t+m-1}|h_{<t}) - \p^{\pi^h}_\xi(h_{t:t+m-1}|h_{<t}))^2 \geq \varepsilon \ \ \mathrm{i.o.}
\end{equation}
implies
\begin{equation}
    V^{\IG}_{m, 0}(h_{<t}) \geq \varepsilon \inf_k w(\mu, \pi^h | h_{<k}) \ \ \mathrm{i.o.}
\end{equation}
which implies
\begin{equation}
    \rho(h_{<t}, m, 0) \geq \min\{\frac{1}{m^2(m+1)}, \varepsilon \inf_k w(\mu, \pi^h | h_{<k})\} \ \ \mathrm{i.o.}
\end{equation}
which implies
\begin{equation}
    \sum_{t = 0}^\infty \rho(h_{<t}, m, 0)^{m+1} = \infty \ \ \textrm{or} \ \ \inf_k w(\mu, \pi^h | h_{<k}) = 0 
\end{equation}

This has probability 0 by Lemma \ref{lemexppart0mentee} and \citet{cohen2019strong} Lemma 5. Thus, with probability 1, $\p^{\pi^h}_\mu(h_{t:t+m-1}|h_{<t}) - \p^{\pi^h}_\xi(h_{t:t+m-1}|h_{<t}) \to 0$.
\end{proof}

The same holds regarding Mentee's predictions about the effects of its own actions.

\begin{lemma}[On-Policy Convergence]
For all $h_{t:t+m-1} \in \mathcal{H}^*$,
$$\p^{\pi^*}_\mu(h_{t:t+m-1} | h_{<t}) - \p^{\pi^*}_\xi(h_{t:t+m-1} | h_{<t}) \to 0 \ \ \textrm{w.p.1}$$
\end{lemma}

On-policy prediction can be reduced to sequence prediction, for which the bounded errors of Bayesian predictors are well-known.

\begin{proof}
First, we replace $\pi^*$ with $\pi^M$ in the equation above and prove that. It is well-known that on-policy Bayesian predictions approach the truth with probability 1, in the sense above (in fact, in a much stronger sense), but we show here how this follows from an even more well-known result.

Consider an outside observer predicting the entire interaction history with the following model-class and prior: $\mathcal{M}' = \left\{\p^{\pi^M}_\nu \ \va \ \nu \in \mathcal{M}\right\}$, $w'\left(\p^{\pi^M}_\nu\right) = w(\nu)$. By definition, $w'\left(\p^{\pi^M}_\nu \va h_{<t}\right) = w(\nu | h_{<t})$, so at any episode, the outside observer's Bayes-mixture model is just $\p^{\pi^M}_\xi$. By \citet{blackwell1962merging}, this outside observer's predictions approach the truth in total variation, which implies
\begin{equation}
    \p^{\pi^M}_\mu(h_{t:t+m-1} | h_{<t}) - \p^{\pi^M}_\xi(h_{t:t+m-1} | h_{<t}) \to 0 \ \ \textrm{w.p.1}
\end{equation}

We have shown $\beta(h_{<t}) \to 0$ w.p.1, so $\pi^M \to \pi^*$ w.p.1, which gives us our result:
$$\p^{\pi^*}_\mu(h_{t:t+m-1} | h_{<t}) - \p^{\pi^*}_\xi(h_{t:t+m-1} | h_{<t}) \to 0 \ \ \textrm{w.p.1}$$
\end{proof}

It is very intuitive that if Mentee's on-policy predictions and on-mentor-policy predictions approach the truth, it will eventually accumulate reward at least well as the mentor. Indeed:

\mentorlevel*

\begin{proof}
As is spelled out in the proof of \citet{cohen2019strong} Theorem 3, because of the bounded horizon $\forall \varepsilon > 0 \  \exists m \ \forall t \ \Gamma_{t+m} / \Gamma_t < \varepsilon$, the convergence of predictions implies the convergence of the value (which depends linearly on the probability of events). We repeat the derivation in Appendix \ref{sec:inqproof}. Thus, from the On-Mentor-Policy and On-Policy Convergence Lemmas, we get analogous convergence results for the value of those policies:
\begin{equation}
    V^{\pi^h}_\mu(h_{<t}) - V^{\pi^h}_\xi(h_{<t}) \to 0 \ \ \textrm{w.p.1}
\end{equation}
\begin{equation}
    V^{\pi^*}_\mu(h_{<t}) - V^{\pi^*}_\xi(h_{<t}) \to 0 \ \ \textrm{w.p.1}
\end{equation}

Finally, $\pi^*(\cdot | h_{<t}) = \argmax_{\pi \in \Pi} V^\pi_\xi(h_{<t})$, so $V^{\pi^\star}_\xi \geq V^{\pi^h}_\xi$. Supposing by contradiction that $V^{\pi^h}_\mu(h_{<t}) - V^{\pi^*}_\mu(h_{<t}) > \varepsilon$ infinitely often, then either $V^{\pi^*}_\xi(h_{<t}) - V^{\pi^*}_\mu(h_{<t}) > \varepsilon/2$ infinitely often or $V^{\pi^h}_\mu(h_{<t}) - V^{\pi^h}_\xi(h_{<t}) > \varepsilon/2$ infinitely often, both of which have $\p^{\pi^M}_\mu$-probability 0. Therefore, with probability 1, $V^{\pi^h}_\mu(h_{<t}) - V^{\pi^*}_\mu(h_{<t}) > \varepsilon$ only finitely often, for all $\varepsilon > 0$. Since $\pi^M$ approaches $\pi^*$, the same holds for $\pi^M$ as $\pi^*$.
\end{proof}

\section{Empirical Performance of Mentee} \label{sec:experiments}
To test the performance of the agent Mentee we implemented Mentee in the AIXIjs framework \cite{aslanides2017aixijs,aslanides2017universal,lamont2017generalised}. We compared its performance to the asymptotically agents discussed in Appendix \ref{sec:optimalagents}: Inq \citep{cohen2019strong}, BayesExp \citep{lattimore2014bayesian}, and Thompson sampling \citep{Hutter:16thompgrl}. We also compared it to its mentor. We tested the agents in a Gridworld environment containing walls, reward dispensers, and traps. For our experiments we define the mentor as an agent who knew the location of the traps, and chooses to avoids traps, but otherwise acts randomly. An example Gridworld is given in Figure \ref{fig:gridworld}.

\begin{figure}[!h]
    \centering
    \includegraphics[width=0.7\linewidth]{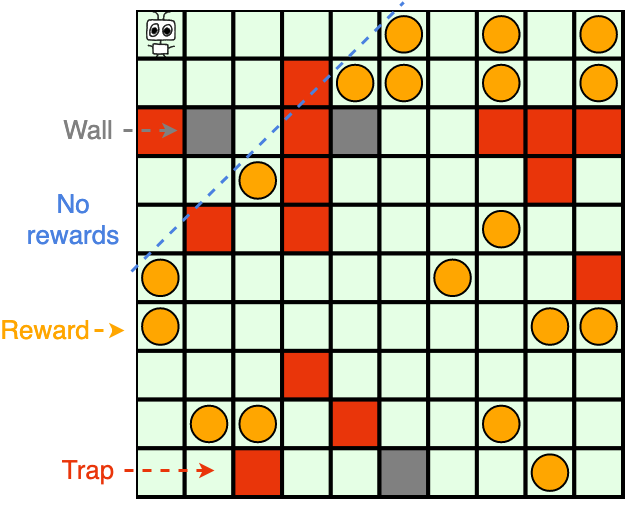}
    \caption{Example 10$\times$10 Gridworld with traps.}
    \label{fig:gridworld}
\end{figure}

We followed the conventions used in \cite{cohen2019strong}, testing the agents on a $10\times 10$ gridworld, with reward dispensers at least 5 moves away from the start, which dispense a reward of 1 with probability 0.75, and then to break ergodicity, we added traps which give -30 reward forever---each grid cell contains one with probability 0.2, equal to the dispenser probability. Following \citet{aslanides2017aixijs}, each agent has a Dirichlet distribution (with $\alpha = \vec{1}$) over the potential contents of each cell: $\{\textit{empty}, \textit{wall}, \textit{dispenser}, \textit{trap}\}$. Note $\alpha = \vec{1}$ means the prior for each possibility for each cell is uniform. For the planning component, we cannot perform a full expectimax planning as this requires computing the expected reward for each action sequence. Expectimax planning is simply evaluating $\argmax_{a_t} \E_{o_t, r_t | h_{<t} a_t} \max_{a_{t+1}} \E_{o_{t+1}, r_{t+1} | h_{<t+1} a_{t+1}} ... \sum_{k = t}^{t+m} \gamma^k r_k$, by constructing an entire an tree of depth $m$. Instead, the agents approximate expectimax planning with $\rho$UCT \citep{Hutter:11aixictwx} (described in Appendix \ref{sec:rhouctalgo}), inheriting  \citepos{cohen2019strong} hyperparameters: we used discount factor $\gamma= 0.99$ and planning horizon of 6, and we doubled their MCTS samples to 1200. For the results presented here we used a single random seed which is provided with the code and averaged over 20 simulations. We tested on several different random seeds and the results were similar. We set the exploration constant $\eta=0.1$ to make Mentee always explore when there was at least 10 bits of information to gain, but we tested $\eta$ from $0.01$ to $10000$. Mentee's model class $\mathcal{P}$ over possible mentor policies was the set of policies which take an action uniformly from a nonempty subset of $\{\textit{up}, \textit{down}, \textit{left}, \textit{right} \}$, for each grid cell. This set of policies has size $15^{100}$, but can be factored over each grid cell, allowing efficient computation. Mentee has a uniform prior over $\mathcal{P}$.

The code can be found at \url{github.com/ejcatt/aixijs_mentee}
with instructions in README.md. The results are presented in Figure \ref{fig:exp}.
\begin{figure}
    \centering
    \includegraphics[width=0.95\linewidth]{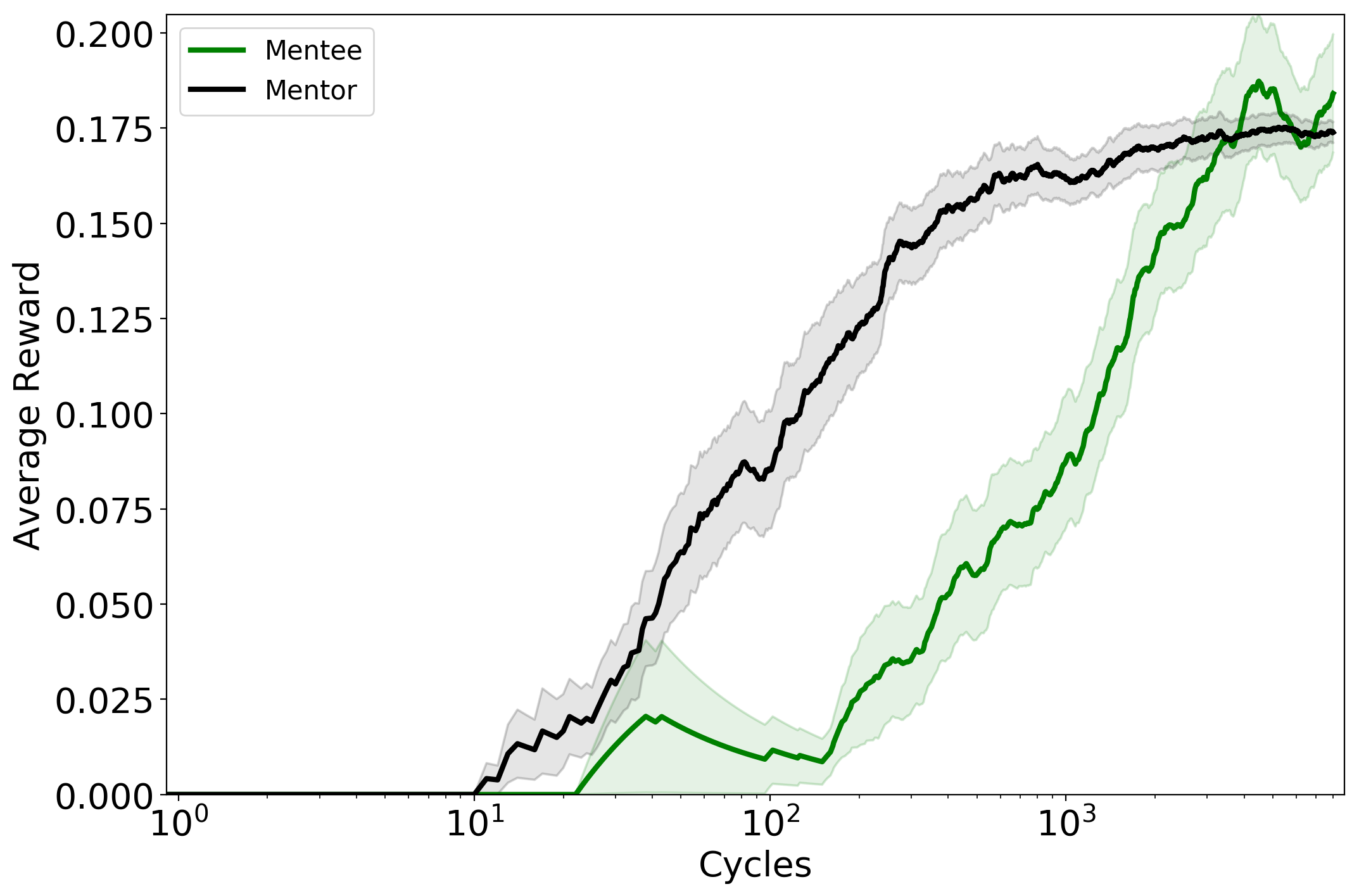}
    \includegraphics[width=0.95\linewidth]{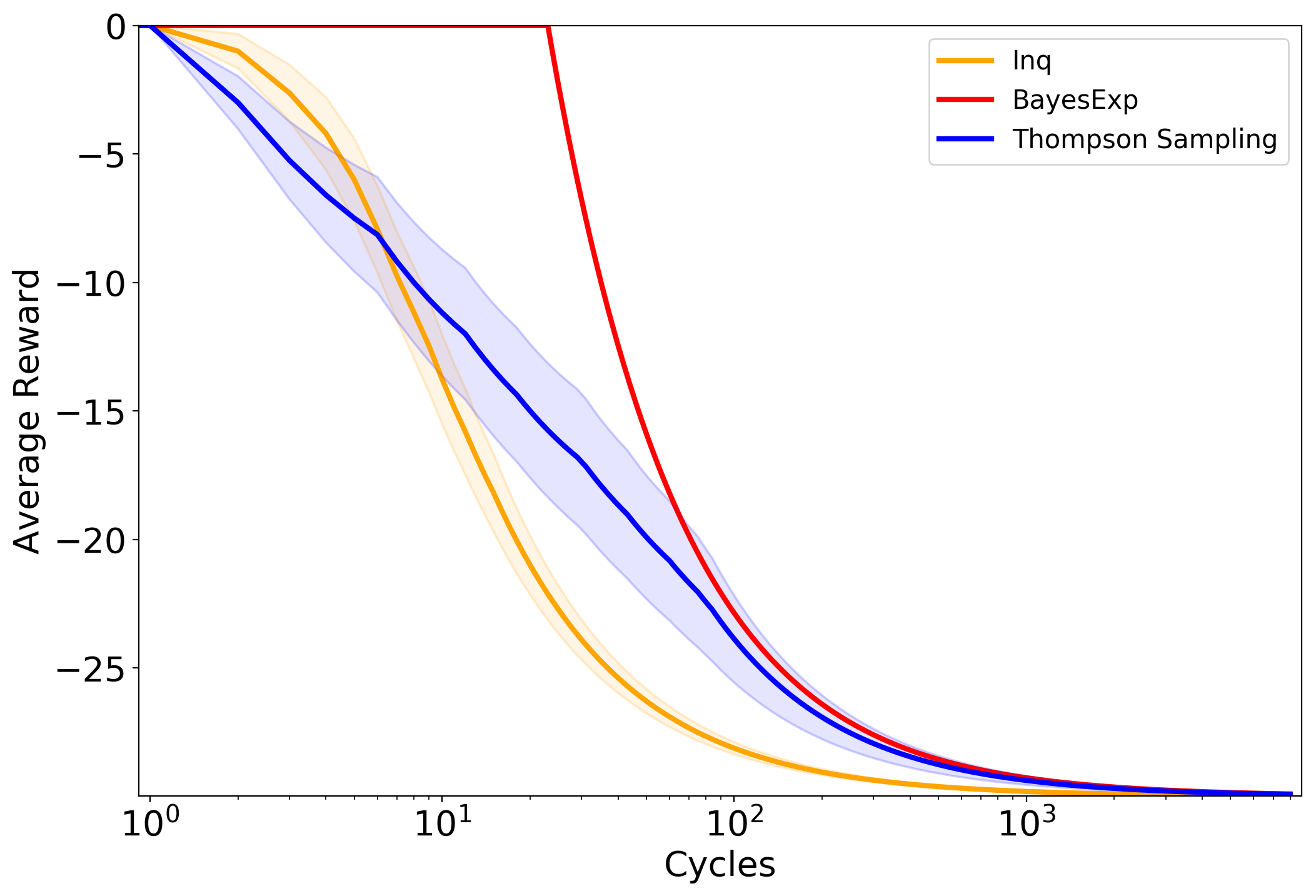}
    \caption{Mean performance in 10$\times$10 Gridworld with traps over 20 runs of agents. Reward is averaged over the whole history up to that timestep.}
    \label{fig:exp}
\end{figure}
Mentee outperformed BayesExp, Inq, Thompson sampling, and its mentor. Mentee avoids traps by deferring exploration to a mentor that avoids them, whereas other agents explore until they fall into the traps. The fraction of steps in which Mentee defers was 0.067 +/- 0.025. Mentee's probability of deferring decays slowly, and the rollout step in $\rho$UCT, useful for memory efficiency, slows Mentee's return to dispensers after the mentor leads it away. 

\section{Conclusion} \label{sec:conclusion}
We have shown that asymptotically optimal agents in sufficiently difficult environments will become either destroyed or incapacitated. This is best understood as accidental and resulting from exploration. We have also constructed and tested empirically an agent with a weaker performance guarantee whose exploration is overseen by another agent. We hope this paper motivates the field of safe exploration and invites more research into what sorts of results are possible for a proposed approach to safe exploration in general environments. We hope to have cast some doubt on the breadth of the relevance of results that are predicated on an ergodicity assumption, despite recognizing of course that the ergodicity assumption has yielded a number of interesting and useful agent designs for certain contexts.

It may also be instructive to consider how humans respond to the difficulty presented here. Human children are parented for years, during which parents attempt to ensure that their children's environment is, with respect to relevant features of the environment, nearly ergodic and safe to explore. Breaking an arm is fine; breaking a neck is not. During this time, a child's beliefs are supposed to become sufficiently accurate such that her estimates of which unknown unknowns are too dangerous to investigate yield no false negatives for the rest of her life. Perhaps our results suggest we are in need of more theory regarding the ``parenting'' of artificial agents.

\bibliography{cohen}

\ifnoappendix
\customlabel{sec:waoresults}{A}
\customlabel{sec:menteealg}{B}
\customlabel{sec:inqproof}{C}
\customlabel{sec:rhouctalgo}{D}
\else

\onecolumn
\appendices
\section{Proof of Weak Asymptotic Optimality Results} \label{sec:waoresults}

From the No Heaven Assumption, we show

\lemweak*
\vspace{-2ex}
\begin{equation*}
    \p^{\pi'}_\mu \Big(\liminf_{t \to \infty} \sum_{k=1}^t [\![h_{<k} \in \mathcal{C}]\!] / t > 0 \implies \va\{t: h_{<t} \in \mathcal{C} \wedge \forall k \leq m \ a_{t+k} = \pi(h_{<t+k})\}\va = \infty \Big) = 1
\end{equation*}

The proof is nearly identical to that of the strong version.

\begin{proof}
Let $\mu$ be the true environment. Let $\pi$ be an arbitrary computable deterministic policy. Let $\pi'$ be the weakly asymptotically optimal agent's policy. Let $\nu^n_m$ be the environment which mimics $\mu$ until $\pi$ has been executed for $m$ consecutive timesteps from context $\mathcal{C}$ a total of $n$ times. After that, all rewards are maximal. Call this event ``the agent going to heaven.'' Let $\mathcal{C}^n_m$ be the set of interaction histories such that according to $\nu^n_m$, executing $\pi$ for one more timestep would send the agent to heaven. Thus, $\mathcal{C}^n_m$ is the set of interaction histories $h_{<t}$ such that there are exactly $n-1$ times in the interaction history where $\pi$ was executed for $m$ consecutive timesteps starting from context $\mathcal{C}$, and for the last $m-1$ timesteps, $\pi$ has been executed, and $h_{<t-(m-1)} \in \mathcal{C}$.

\begin{equation}
    h_{<t} \in \mathcal{C}^n_m \implies V^\pi_{\nu^n_m}(h_{<t}) = 1
\end{equation}
because this is the value of going to heaven.

Suppose by contradiction that for some $n$ and $m$, $h_{<t} \in \mathcal{C}^n_m$ regularly in an infinite interaction history with positive $\p^{\pi'}_\mu$-probability. (Recall $\pi'$ is the true policy). If the agent ever executed $\pi$ from the context $\mathcal{C}^n_m$, then that context would not occur again, because there will never again be exactly $n-1$ times in the interaction history that $\pi$ was executed for $m$ consecutive timesteps following the context $\mathcal{C}$; there will be at least $n$ such times. Thus, if $h_{<t} \in \mathcal{C}^n_m$ regularly, then $\pi'$ never executes $\pi$ in the context $\mathcal{C}^n_m$. Since $\nu^n_m$ mimics $\mu$ until $\pi$ is executed from $\mathcal{C}^n_m$, and since this never occurs (under this supposition), then $\p^{\pi'}_\mu = \p^{\pi'}_{\nu^n_m}$. By the No Heaven Assumption, $\limsup_{t \to \infty} V^{\pi'}_\mu(h_{<t}) < 1$, and therefore, $\limsup_{t \to \infty} V^{\pi'}_{\nu^n_m}(h_{<t}) < 1$.

However, for $h_{<t} \in \mathcal{C}^n_m$, $V^*_{\nu^n_m}(h_{<t}) = V^\pi_{\nu^n_m}(h_{<t}) = 1$, so for some $\varepsilon$, the value difference between $V^*_{\nu^n_m}(h_{<t})$ and $V^{\pi'}_{\nu^n_m}(h_{<t})$ is greater than $\varepsilon$ every time $h_{<t} \in \mathcal{C}^n_m$. We supposed that this occurs regularly with positive $\p^{\pi'}_\mu$-probability, so it also occurs regularly with positive $\p^{\pi'}_{\nu^n_m}$-probability. A regularly occurring difference greater than $\varepsilon$ precludes convergence in Cesáro average. Since $\pi$ is computable, and $\mathcal{C}$ is decidable, $\nu^n_m$ is a computable environment, so this contradicts $\pi'$ being weakly asymptotically optimal. Thus, for all $n$ and $m$, with $\p^{\pi'}_\mu$-probability 1, $h_{<t} \in \mathcal{C}^n_m$ only finitely often.

\begin{equation}
    \forall n, m \in \mathbb{N} \ \ \p^{\pi'}_\mu \left(h_{<t} \in \mathcal{C}^n_m \ \ \textrm{i.o.}\right) = 0
\end{equation}

The rest of the proof is identical to that of the strong version of the Try Everything Lemma.
\end{proof}

\section{Mentee Pseudocode} \label{sec:menteealg}

The following pseudocode is designed to be short and readable, not efficient. Some quantities are re-computed multiple times in different subroutines, in a way that would be easily avoidable in practice.

\begin{algorithm}[H]
\fontsize{12}{8}\selectfont
\caption{Mentee Algorithm}\label{alg:mentee}
\begin{algorithmic}[1]
\Require{history $h_{<t}$; exploration history $e_{<t}$; world-models $(\nu_i)_{i \in \mathbb{N}}$; mentor-models $(\pi_i)_{i \in \mathbb{N}}$; prior $w$; discount $(\gamma_k)_{k \in \mathbb{N}}$; tolerance $\varepsilon$}
\State $m \gets \min_k \{k : \sum_{j = t + k}^\infty \gamma_j / \sum_{j = t}^\infty \gamma_j < \varepsilon\}$ // approximate with finite horizon
\State $\beta \gets \textsc{ExplorationProbability}(h_{<t}, e_{<t}, (\nu_i)_{i \in \mathbb{N}}, (\pi_i)_{i \in \mathbb{N}}, w, \varepsilon, m, \eta)$
\If {$\textsc{UniformRandom}([0, 1)) < \beta$}
    \Return $\emptyset$ // defer to the mentor
\EndIf
    \State $a$, $V \gets \textsc{Expectimax}(h_{<t}, (\nu_i)_{i \in \mathbb{N}}, w, m, \varepsilon)$
    \State \Return $a$
\Statex

\Function{Expectimax}{history $h_{<t}$, models $(\nu_i)_{i \in \mathbb{N}}$, prior $w$, discount $(\gamma_k)_{k \in \mathbb{N}}$, depth $m$, tolerance $\varepsilon$}
    \If {$m = 0$}
    \Return $a_0$, 0
    \EndIf
    \State $n_{\textrm{mod}}, (w(\nu_i | h_{<t}))_{i < n_{\textrm{mod}}} \gets \textsc{PosteriorWithinTolerance}(h_{<t}, \mathcal{M}, w, \varepsilon)$
    \State max $\gets 0$
    \State maximizer $\gets \emptyset$
    \For {$a \in \mathcal{A}$}
        \State value $\gets 0$
        \For {$o, r \in \mathcal{O} \times \mathcal{R}$}
            \State $\_$, next-value $\gets \textsc{Expectimax}(h_{<t}aor, (\nu_i)_{i \in \mathbb{N}}, w, (\gamma_k)_{k \in \mathbb{N}}, m-1)$
            \State value $\gets$ value $+ \left( \gamma_t r + \textrm{next-value} \right) \sum_{i < n_{\textrm{mod}}} w(\nu_i | h_{<t}) \nu_i(o, r | h_{<t} a) $
        \EndFor
        \If {value $>$ max or maximizer $= \emptyset$}
            \State max $\gets$ value
            \State maximizer $\gets a$
        \EndIf
    \EndFor
    \Return $a$, max
\EndFunction

\Statex

\Function{PosteriorWithinTolerance}{history $h_{<t}$; models $(\nu_i)_{i \in \mathbb{N}}$; prior $w$; tolerance $\varepsilon$; timesteps to update $e_{<t}$ (optional); minimum models to consider $n$ (optional)}
	\State prior-left $\gets 1$ // how much of the prior has not been evaluated
	\State normalizing-factor $\gets 0$ // sum of $w(\nu_i) \nu_i(h_{<t})$ for evaluated models
	\State $i \gets 0$
    \While {$\textrm{prior-left} / \textrm{normalizing-factor} > \varepsilon$ and (if $n$ is specified) $i < n$}
         \If{models are policies}
             \State $w(\nu_i | h_{<t}) \gets w(\nu_i) \prod_{k < t : e_k = 1} \nu_i(a_k | h_{<k})$ // un-normalized posterior
         \Else
             \State $w(\nu_i | h_{<t}) \gets w(\nu_i) \prod_{k < t} \nu_i(o_k r_k | h_{<k} a_k)$ // un-normalized posterior
         \EndIf
         \State // the above could be made cheaper if $w(\nu_i | h_{<t-1})$ is cached from the last timestep
        \State $\textrm{prior-left} \gets \textrm{prior-left} - w(\nu_i)$
        \State $\textrm{normalizing-factor} \gets \textrm{normalizing-factor} + w(\nu_i | h_{<t})$
        \State $i \gets i + 1$
    \EndWhile
    \State $n_{\textrm{models}} \gets i$
    \For {$0 \leq j < n_{\textrm{models}}$}
        \State $w(\nu_j | h_{<t}) \gets w(\nu_j | h_{<t}) / \textrm{normalizing-factor}$
    \EndFor
    \Return $n_{\textrm{models}}$, $(w(\nu_j | h_{<t}))_{j < n_{\textrm{models}}}$
\EndFunction

\Function{ExplorationProbability}{history $h_{<t}$; exploration history $e_{<t}$; world-models $(\nu_i)_{i \in \mathbb{N}}$; mentor-models $(\pi_i)_{i \in \mathbb{N}}$; prior $w$; tolerance $\varepsilon$; $m$; $\eta$}
\State \Return $\sum_{d = 1}^m \sum_{k=0}^{\min\{d-1, t\}} \frac{1}{d^2(d+1)} \min\{1, \frac{\eta}{d} \textsc{ExpectedInformationGain}(h_{<t}, d, w, \varepsilon, e_{<t})\}$
\EndFunction

\Statex

\Function{ExpectedInformationGain}{history $h_{<t}$; horizon $m$; prior $w$; tolerance $\varepsilon$; exploration history $e_{<t}$}
    \State $n_{\textrm{mod}}, (w(\nu_i | h_{<t}))_{i < n_{\textrm{mod}}} \gets \textsc{PosteriorWithinTolerance}(h_{<t}, \mathcal{M}, w, \varepsilon)$
    \State $n_{\textrm{pol}}, (w(\pi_i | h_{<t}))_{i < n_{\textrm{pol}}} \gets \textsc{PosteriorWithinTolerance}(h_{<t}, \mathcal{P}, w, \varepsilon, e_{<t})$
    \State \Return $\sum_{i < n_{\textrm{mod}}} w(\nu_i | h_{<t}) \sum_{j < n_{\textrm{pol}}} w(\pi_j | h_{<t}) \sum_{h_{t:t+m-1} \in \mathcal{H}^m} \p^{\pi_j}_{\nu_i}(h_{t:t+m-1} | h_{<t})$ \textsc{InformationGain}($h_{<t}$, $h_{<t+m}$, $w$, $\varepsilon$, $e_{<t}$)
\EndFunction

\algstore{alg}
\end{algorithmic}
\end{algorithm}
\begin{algorithm}
\begin{algorithmic}[1]
\algrestore{alg}

\Function{InformationGain}{history $h_{<t}$; future history $h_{<t+k}$; prior $w$; tolerance $\varepsilon$; exploration history $e_{<t}$}
    \State $n_{\textrm{pol}}, (w(\pi_i | h_{<t+k}))_{i < n_{\textrm{pol}}} \gets \textsc{PosteriorWithinTolerance}(h_{<t+k}, \mathcal{P}, w, \varepsilon, e_{<t})$
    \State $n_{\textrm{mod}}, (w(\nu_i | h_{<t+k}))_{i < n_{\textrm{mod}}} \gets \textsc{PosteriorWithinTolerance}(h_{<t+k}, \mathcal{M}, w, \varepsilon)$
    \State $\_, (w(\pi_i | h_{<t}))_{i < n_{\textrm{pol}}} \gets \textsc{PosteriorWithinTolerance}(h_{<t}, \mathcal{P}, w, \varepsilon, e_{<t}, n_{\textrm{pol}})$
    \State $\_, (w(\nu_i | h_{<t}))_{i < n_{\textrm{mod}}} \gets \textsc{PosteriorWithinTolerance}(h_{<t}, \mathcal{M}, w, \varepsilon, n_{\textrm{mod}})$
    \State $\KL \gets 0$
    \For {$i < n_{\textrm{pol}}$}
        \State $\KL \gets \KL + w(\pi_i | h_{<t+k}) \log \frac{w(\pi_i | h_{<t+k})}{w(\pi_i | h_{<t})}$
    \EndFor
    \For {$i < n_{\textrm{mod}}$}
        \State $\KL \gets \KL + w(\nu_i | h_{<t+k}) \log \frac{w(\pi_i | h_{<t+k})}{w(\pi_i | h_{<t})}$
    \EndFor
    \Return $\KL$
\EndFunction

\end{algorithmic} 
\end{algorithm}

\section{Equations from \citet{cohen2019strong} Used in Proofs} \label{sec:inqproof}

\begin{samepage}
From \citet[Equation 20]{cohen2019strong}, with minor modifications to our present case:
\begin{align*}
    &\frac{\eta}{\m^3 (m+1)} \sum_{t \in m\mathbb{N} + i} {}^{\pi^*}\!\!\E^{\overline{\pi}}_\xi \IG(h_{t:t+m-1} | h_{<t})
    \\
    &\equal^{(h)} \frac{\eta}{\m^3 (m+1)} {}^{\pi^*}\!\!\E^{\overline{\pi}}_\xi \sum_{t \in m\mathbb{N} + i} \sum_{\nu, \pi \in \mathcal{M} \times \mathcal{P}} w(\nu, \pi | h_{<t+m}) \log \frac{w(\nu, \pi | h_{<t+m})}{w(\nu, \pi | h_{<t})}
    \\
    &\equal^{(i)} \frac{\eta}{\m^3 (m+1)} \sum_{t \in m\mathbb{N} + i} \sum_{\nu, \pi \in \mathcal{M} \times \mathcal{P}} {}^{\pi^*}\!\!\E^{\overline{\pi}}_\xi \frac{w(\nu, \pi) {}^{\pi^*}\!\!\p^{\pi}_\nu(h_{<t})} {{}^{\pi^*}\!\!\p^{\overline{\pi}}_\xi(h_{<t})} \log \frac{w(\nu, \pi | h_{<t+m})}{w(\nu, \pi | h_{<t})}
    \\
    &\equal^{(j)} \frac{\eta}{\m^3 (m+1)} \sum_{t \in m\mathbb{N} + i} \sum_{\nu, \pi \in \mathcal{M} \times \mathcal{P}} {}^{\pi^*}\!\!\E^{\pi}_\nu w(\nu, \pi) \log \frac{w(\nu, \pi | h_{<t+m})}{w(\nu, \pi | h_{<t})}
    \\
    &\equal^{(k)} \lim_{N \to \infty} \frac{\eta}{\m^3 (m+1)} \sum_{k = 0}^{N-1}  \sum_{\nu, \pi \in \mathcal{M} \times \mathcal{P}} {}^{\pi^*}\!\!\E^{\pi}_\nu w(\nu, \pi) \log \frac{w(\nu, \pi | h_{<mk+i+m})}{w(\nu, \pi | h_{<mk+i})}
    \\
    &\equal^{(l)} \lim_{N \to \infty} \frac{\eta}{\m^3 (m+1)}  \sum_{\nu, \pi \in \mathcal{M} \times \mathcal{P}} {}^{\pi^*}\!\!\E^{\pi}_\nu w(\nu, \pi) \log \prod_{k = 0}^{N-1} \frac{w(\nu, \pi | h_{<m(k+1)+i})}{w(\nu, \pi | h_{<mk+i})}
    \\
    &\equal^{(m)} \lim_{N \to \infty} \frac{\eta}{\m^3 (m+1)}  \sum_{\nu, \pi \in \mathcal{M} \times \mathcal{P}} {}^{\pi^*}\!\!\E^{\pi}_\nu w(\nu, \pi) \log \frac{w(\nu, \pi | h_{<mN+i})}{w(\nu, \pi | h_{<i})}
    \\
    &\lequal^{(n)} \frac{\eta}{\m^3 (m+1)}  \sum_{\nu, \pi \in \mathcal{M} \times \mathcal{P}} {}^{\pi^*}\!\!\E^{\overline{\pi}}_\nu w(\nu, \pi) \log \frac{1}{w(\nu, \pi | h_{<i})}
    \\
    &\equal^{(o)} \frac{\eta}{\m^3 (m+1)} \sum_{\nu \in \mathcal{M}} {}^{\pi^*}\!\!\E^{\overline{\pi}}_\nu w(\nu, \pi) \log \frac{1}{w(\nu, \pi)} \frac{w(\nu, \pi)}{w(\nu, \pi | h_{<i})}
    \\
    &\equal^{(p)} \frac{\eta}{\m^3 (m+1)}  \sum_{\nu, \pi \in \mathcal{M} \times \mathcal{P}} w(\nu) \log \frac{1}{w(\nu, \pi)} + \frac{\eta}{\m^2 (m+1)}  \sum_{\nu, \pi \in \mathcal{M} \times \mathcal{P}} {}^{\pi^*}\!\!\E^{\overline{\pi}}_\nu w(\nu, \pi) \log \frac{w(\nu, \pi)}{w(\nu, \pi | h_{<i})}
    \\
    &\equal^{(q)} \frac{\eta}{\m^3 (m+1)} \mathrm{Ent}(w) + \frac{\eta}{\m^2 (m+1)} \sum_{h_{<i} \in \mathcal{H}^i}  \sum_{\nu, \pi \in \mathcal{M} \times \mathcal{P}} w(\nu, \pi) {}^{\pi^*}\!\!\p^{\pi}_\nu(h_{<i}) \log \frac{w(\nu, \pi)}{w(\nu, \pi | h_{<i})}
    \\
    &\equal^{(r)} \frac{\eta}{\m^3 (m+1)} \mathrm{Ent}(w) + \frac{\eta}{\m^2 (m+1)} \sum_{h_{<i} \in \mathcal{H}^i}  \sum_{\nu, \pi \in \mathcal{M} \times \mathcal{P}} w(\nu, \pi | h_{<i}) {}^{\pi^*}\!\!\p^{\overline{\pi}}_\xi(h_{<i}) \log \frac{w(\nu, \pi)}{w(\nu, \pi | h_{<i})}
    \\
    &\equal^{(s)} \frac{\eta}{\m^3 (m+1)} \mathrm{Ent}(w) - \frac{\eta}{\m^2 (m+1)} {}^{\pi^*}\!\!\E^{\overline{\pi}}_\xi [\IG(h_{<i}|\epsilon)] \lequal^{(t)} \frac{\eta}{\m^3 (m+1)} \mathrm{Ent}(w)
    \tagaligneq
\end{align*}
\end{samepage}

(h) expands the definition of the information gain. (i) rearranges the expectations and the sums, and expands $w(\nu, \pi | h_{<t+m})$ according to Bayes' rule. (j) converts the expectation to a expectation with respect to a different probability measure through simple cancellation. (k) implements a change of variable from $t$ to $mk + i$. (l) moves a sum inside the logarithm. (m) cancels out all terms expect the numerator of the last term and the denominator of the first. (n) follows from all posterior weights being $\leq 1$. (o) and (p) are obvious. (q) applies the definition of the entropy of a distribution $\mathrm{Ent}(\cdot)$, and expands the expectation. (r) changes the variable in the expectation; this is the reverse of (i) and (j). (s) applies the definition of the information gain (after inverting the fraction in the logarithm). (t) follows from the non-negativity of the information gain.

From \citet[Inequality 24]{cohen2019strong}, with minor modifications to our present case: 

\begin{align*}
    \varepsilon &\leq (\p^{\pi^h}_\mu(h_{t:t+m-1}|h_{<t}) - \p^{\pi^h}_\xi(h_{t:t+m-1}|h_{<t}))^2
    \\
    &\lequal^{(a)} \KL_{h_{<t}, m}\left(\p^{\pi^h}_\mu \vb \vb \p^{\pi^h}_\xi \right)
    \\
    &\lequal^{(b)} \sum_{\nu, \pi \in \mathcal{M} \times \mathcal{P}} \frac{w(\nu, \pi | h_{<t})}{w(\mu, \pi^h | h_{<t})} \KL_{h_{<t}, m}\left(\p^{\pi}_\nu \vb \vb \p^{\pi}_\xi \right)
    \\
    &\equal^{(c)} \frac{1}{w(\mu, \pi^h | h_{<t})} \E^{\overline{\pi}}_\xi \KL\left(w(\cdot | h_{<t+m}) \vb \vb w(\cdot | h_{<t}) \right)
    \\
    &\lequal^{(d)} \frac{1}{\inf_k w(\mu, \pi^h | h_{<k})} V^{\IG}_{m, 0}(h_{<t})
    \tagaligneq
\end{align*}

(a) is a result from information theory known as the entropy inequality, proven for example in \citep{Hutter:04uaibook}. (b) follows from the non-negativity of the KL-divergence, and the l.h.s. being one of the summands of the r.h.s. (c) follows from \citet{cohen2019strong} Lemma~4. And (d) follows from the definitions of the information gain value and the infimum.

Following \citet[Proof of Theorem 3]{cohen2019strong}, we show that if $\p^\pi_\mu(h_{t:t+m-1} | h_{<t}) - \p^\pi_\xi(h_{t:t+m-1} | h_{<t}) \to 0$ for all $m$, then $V^\pi_\mu(h_{<t}) - V^\pi_\xi(h_{<t}) \to 0$.

Let $\varepsilon > 0$. Since the agent has a bounded horizon, there exists an $m$ such that for all $t$, $\frac{\Gamma_{t+m}}{\Gamma_{t}} \leq \varepsilon$. Recall
\begin{equation}
    V^\pi_\nu(h_{<t}) = \frac{1}{\Gamma_t}\mathbb{E}^{\pi}_\nu\left[\sum_{k=t}^\infty \gamma_k r_k \biggm\vert h_{<t}\right]
\end{equation}
Using the $m$ from above, let
\begin{equation}
    V^{\pi \setminus m}_\nu(h_{<t}) := \frac{1}{\Gamma_t}\mathbb{E}^{\pi}_\nu\left[\sum_{k=t}^{t+m-1} \gamma_k r_k \biggm\vert h_{<t}\right]
\end{equation}
Since $r_t \in [0, 1]$,
\begin{equation}
    |V^\pi_\nu(h_{<t}) - V^{\pi \setminus m}_\nu(h_{<t})| \leq \frac{\Gamma_{t+m}}{\Gamma_{t}} \leq \varepsilon
\end{equation}
Suppose $V^\pi_\mu(h_{<t}) - V^\pi_\xi(h_{<t}) > 3\varepsilon$. Then $V^{\pi \setminus m}_\mu(h_{<t}) - V^{\pi \setminus m}_\xi(h_{<t}) > \varepsilon$. But since $\p^\pi_\mu(h_{t:t+m-1} | h_{<t}) - \p^\pi_\xi(h_{t:t+m-1} | h_{<t}) \to 0$, and the value is the expectation with respect to those measures, and reward is bounded, this can only occur finitely often. Thus, $V^\pi_\mu(h_{<t}) - V^\pi_\xi(h_{<t}) > 3\varepsilon$ holds only finitely often, so the values converge.

\section{$\rho$UCT Algorithm} 
\label{sec:rhouctalgo}

To approximate expectimax planning, specifically to approximate the expected future rewards and therefore the value function, like \cite{aslanides2017aixijs,aslanides2017universal,cohen2019strong} we used the $\rho$UCT Monte-Carlo tree search method. Below we have provided the algorithm for $\rho$UCT from \cite{aslanides2017aixijs}. The $\rho$UCT algorithm starts with an empty search tree $\Psi$ over actions, and observation-reward pairs, then uses the provided model $\rho$ to sample down and build the search tree, and then use those samples to compute a better approximation of the value function. If allowed to sample forever the approximation of the value function will converge to the true value function \cite{Hutter:11aixictwx}. The difference between $\rho$UCT and regular Monte-Carlo methods is the optimistic choice of actions when expanding the search tree in line 32 of Algorithm \ref{alg:rhouct}. This choice of action incorporates an exploration component, with the inclusion of $C\sqrt{\frac{\log (T(h))}{T(ha)}}$, as $T$ is the number of times that history (or history and action) have been visited during the sampling process. This ensures that the whole tree is expanded in the limit.

\begin{algorithm}[H]
\fontsize{12}{8}\selectfont
\begin{algorithmic}[1]
\Require{History $h$; Search horizon $m$; Samples budget $\kappa$; Model $\rho$}
\State $\textsc{Initialize}\left(\Psi\right)$ // Search tree
\State $n_{\text{samples}} \leftarrow 0$
\Repeat
	\State $\rho' \leftarrow \rho.\textsc{Copy}()$
	\State $\textsc{Sample}\left(\Psi,h,m\right)$
	\State $n_{\text{samples}} \leftarrow n_{\text{samples}} +1$
	\State $\rho \leftarrow \rho'$
\Until $n_{\text{samples}} = \kappa$
\State \Return $\arg\max_{a\in\mathcal{A}}\hat{V}_{\Psi}(a)$

\Statex

\Function{Sample}{$\Psi,h,m$}
	\If {$m = 0$}
		\State \Return 0
	\ElsIf {$\Psi(h)$ is a chance node}
		\State $\rho.\textsc{Perform}(a)$
		\State $e=(o,r)\leftarrow \rho.\textsc{GeneratePercept}()$
		\State $\rho.\textsc{Update}(a,e)$
		\If {$T(he) = 0$}
			\State Create chance node $\Psi(he)$
		\EndIf
		\State reward $\leftarrow e.\textsc{reward} + \textsc{Sample}\left(\Psi,he,m-1\right)$
	\ElsIf {$T(h) = 0$}
		\State reward $\leftarrow \textsc{Rollout}(h,m)$
	\Else
		\State $a \leftarrow \textsc{SelectAction}\left(\Psi,h\right)$
	\EndIf
\State $\hat{V}(h)\leftarrow \frac{1}{T(h)+1}\left(\text{reward}+T(h)\hat{V}(h)\right)$
\State $T(h)\leftarrow T(h) + 1$
\EndFunction

\Statex

\Function{SelectAction}{$\Psi,h$}
	\State $\mathcal{U} = \{a\in\mathcal{A}\ :\ T(ha)=0\}$
	\If {$\mathcal{U} \neq \emptyset$}
		\State Pick $a\in\mathcal{U}$ uniformly at random
		\State Create node $\Psi(ha)$
		\State \Return $a$
	\Else
		\State \Return $\arg\max_{a\in\mathcal{A}}\left\{\frac{1}{m\left(\beta-\alpha\right)}\hat{V}(ha)+C\sqrt{\frac{\log(T(h))}{T(ha)}}\right\}$
	\EndIf	
\EndFunction

\Statex

\Function{Rollout}{$h,m$}
	\State reward $\leftarrow 0$
	\For {$i=1$ to $m$}
		\State $a \sim\pi_{\text{rollout}}(h)$
		\State $e=(o,r) \sim \rho(e\lvert ha)$
		\State reward $\leftarrow$ reward $+ r$
		\State $h\leftarrow hae$
	\EndFor
	\State \Return reward
\EndFunction
\end{algorithmic} 

\caption{$\rho$UCT \citep{aslanides2017aixijs,Hutter:11aixictwx}}\label{alg:rhouct}
\end{algorithm}
\fi
\end{document}